\documentclass{article}

\PassOptionsToPackage{numbers}{natbib}

\usepackage[preprint]{neurips_2025}

\newcommand{\Acal}{\mathcal{A}}
\newcommand{\R}{\mathbb{R}}
\newcommand{\E}{\mathbb{E}}

\usepackage{subcaption}
\usepackage{algorithm}
\usepackage{algpseudocode}
\usepackage{amsthm}

\newtheorem{theorem}{Theorem}

\newtheorem{proposition}{Proposition}
\newtheorem{definition}{Definition}
\usepackage[utf8]{inputenc} % allow utf-8 input
\usepackage[T1]{fontenc}    % use 8-bit T1 fonts
\usepackage{hyperref}       % hyperlinks
\usepackage{url}            % simple URL typesetting
\usepackage{booktabs}       % professional-quality tables
\usepackage{amsfonts}       % blackboard math symbols
\usepackage{nicefrac}       % compact symbols for 1/2, etc.
\usepackage{microtype}      % microtypography
\usepackage{xcolor}         % colors
\usepackage{graphicx}

\graphicspath{{fig/}}
\usepackage{amsfonts} % For \mathbb command if needed, though not used here.

\usepackage{amsmath, amssymb, amsthm}
\usepackage{bm} % For bold math symbols like {\theta}

\newcommand{\dd}{\mathrm{d}}

\title{Accelerating Optimization via Differentiable \\ Stopping Time}

\author{%
  Zhonglin Xie\\
  Beijing International Center for Mathematical Research\\
  Peking University\\
  \texttt{zlxie@pku.edu.cn} \\
  \And
  Yiman Fong \\
  Department of Industrial Engineering \\
  Tsinghua University \\
  \texttt{fangym23@mails.tsinghua.edu.cn} \\
  \And
  Haoran Yuan \\
  School of Mathematics Science \\
  Peking University \\
  \texttt{yuanhr@stu.pku.edu.cn} \\
  \And
  Zaiwen Wen \\
  Beijing International Center for Mathematical Research\\
  Peking University\\
  \texttt{wenzw@pku.edu.cn} \\
}

\begin{document}

\maketitle

\begin{abstract}
Optimization is an important module of modern machine learning applications. Tremendous efforts have been made to accelerate optimization algorithms. A common formulation is achieving a lower loss at a given time. This enables a differentiable framework with respect to the algorithm hyperparameters. In contrast, its dual, minimizing the time to reach a target loss, is believed to be non-differentiable, as the time is not differentiable. As a result, it usually serves as a conceptual framework or is optimized using zeroth-order methods. To address this limitation, we propose a differentiable stopping time and theoretically justify it based on differential equations. An efficient algorithm is designed to backpropagate through it. As a result, the proposed differentiable stopping time enables a new differentiable formulation for accelerating algorithms. We further discuss its applications, such as online hyperparameter tuning and learning to optimize. Our proposed methods show superior performance in comprehensive experiments across various problems, which confirms their effectiveness.
\end{abstract}

\section{Introduction}
\label{sec:introduction}

Optimization algorithms are fundamental to a wide range of applications, including operations research \cite{taha1997operations}, the training of large language models \cite{zhao2025surveylargelanguagemodels}, and decision-making in financial markets \cite{pilbeam2018finance}. Consequently, significant research effort has been dedicated to accelerating these algorithms. A common formulation for algorithm design and tuning, prevalent in areas like hyperparameter optimization \cite{feurer2019hyperparameter}  and learning to optimize (L2O) \cite{chen2024learning}, is to minimize the objective function value achieved after a fixed number of iterations or a predetermined time budget. This approach often leads to differentiable training objectives with respect to algorithmic hyperparameters, enabling gradient-based optimization of the algorithm itself.

However, this formulation does not directly optimize the number of iterations required to reach a desired performance level or target loss, which is often the practical goal in deployment. The dual objective, minimizing the time to reach a target loss, is traditionally perceived as non-differentiable with respect to algorithm parameters, as stopping time is typically an integer-valued or non-smooth function of parameters, addressed only conceptually or via zeroth-order optimization methods \cite{nemirovskij1983problem}.

To overcome this fundamental challenge and enable the direct, gradient-based optimization of convergence speed towards a target accuracy for iterative algorithms, this paper introduces the concept of differentiable stopping time. We propose a comprehensive framework that allows for the computation of sensitivities of the number of iterations required to reach a stopping criterion with respect to algorithm parameters. Our main contributions are summarized as follows:
\begin{itemize}
    \item We formulate a new class of differentiable objectives for algorithm acceleration, aiming to directly minimize the number of iterations or computational time required to achieve a target performance. This is supported by a theoretical framework that establishes the differentiability of discrete stopping time via a connection between discrete-time iterative algorithms and continuous-time dynamics, leveraging tools from the theory of continuous stopping times.
    \item We develop a memory-efficient and scalable algorithm for computing sensitivities of discrete stopping time, enabling effective backpropagation through iterative procedures. Our experimental results validate the accuracy and efficiency of the proposed method, particularly in high-dimensional settings, and show clear advantages over approaches relying on exact ordinary differential equation solvers.
    \item We demonstrate the applicability of differentiable stopping time in practical applications, including L2O and the online adaptation of optimizer hyperparameters. These case studies show that differentiable stopping time can be seamlessly integrated into existing frameworks, and our empirical evaluations suggest that it provides a principled and effective lens for understanding and improving algorithmic acceleration.
\end{itemize}

\subsection{Related Work}
\label{sec:literature}

\textbf{ODE Perspective of Accelerated Methods.}
Offering a continuous-time view of optimization algorithms, this perspective provides both theoretical insights and practical improvements. The foundational work \cite{su2016differential} established a connection between Nesterov's accelerated gradient method and a second-order ordinary differential equation, introducing a dynamical systems viewpoint for understanding acceleration. Building on this, acceleration phenomena have been analyzed through high-resolution differential equations \cite{shi2022understanding}, revealing deeper insights into optimization dynamics. The symplectic discretization of these high-resolution ODEs \cite{shi2019acceleration} has also been explored, leading to practical acceleration techniques with theoretical guarantees. A Lyapunov analysis of accelerated gradient methods was developed in \cite{laborde2020lyapunov}, extending the framework to stochastic settings. For optimization on parametric manifolds, accelerated natural gradient descent methods have been formulated in \cite{li2025angd}, based on the ODE perspective.

\textbf{Implicit Differentiation in Deep Learning.}
This technique enables efficient gradient computation through complex optimization procedures. A modular framework for implicit differentiation was presented in \cite{blondel2022efficient}, unifying existing approaches and introducing new methods for optimization problems. In non-smooth settings, \cite{bolte2021nonsmooth} developed a robust theory of nonsmooth implicit differentiation with applications to machine learning and optimization. Implicit differentiation has also been applied to train iterative refinement algorithms \cite{chang2022object}, treating object representations as fixed points. For non-smooth convex learning problems, fast hyperparameter selection methods have been developed using implicit differentiation \cite{bertrand2022implicit}. Training techniques for implicit models that match or surpass traditional approaches have been explored in \cite{geng2021training}, leveraging implicit differentiation. Implicit bias in overparameterized bilevel optimization has been investigated in \cite{vicol2022implicit}, providing insights into the behavior of implicit differentiation in high-dimensional settings. In optimal control, implicit differentiation for learning problems has been revisited in \cite{xu2023revisiting}, where new methods for differentiating optimization-based controllers are proposed.

\textbf{Learning to Optimize.}
This emerging paradigm leverages machine learning techniques to design optimization algorithms. A comprehensive overview of L2O methods \cite{chen2022learning} categorizes the landscape and establishes benchmarks for future research. The scalability of L2O to large-scale optimization problems has been explored in \cite{chen2022scalable}, showing that learned optimizers can effectively train large neural networks. To enhance the robustness of learned optimizers, policy imitation techniques were introduced in \cite{chen2020training}, significantly improving L2O model performance. Generalization capabilities have been studied by developing provable bounds for unseen optimization tasks \cite{yang2023learning}. In \cite{yang2023ml2o}, meta-learning approaches are proposed for fast self-adaptation of learned optimizers. The theoretical foundations of L2O have been strengthened through convergence guarantees for robust learned optimization algorithms \cite{song2024towards}. Furthermore, L2O has been extended to the design of acceleration methods by leveraging an ODE perspective of optimization algorithms \cite{xie2024ode}.

\section{Differentiable Stopping Time: From Continuous to Discrete}
\label{sec:diff-stop-time}
We consider an iterative algorithm that arises from the discretization of an underlying continuous-time dynamical system. Let $\Acal({\theta}, x, t)$ be a function that defines the instantaneous negative rate of change for the state $x \in \R^d$ at time $t$, parameterized by ${\theta}$. The input ${\theta}$ could represent, for example, parameters of a step size schedule or weights of a learnable optimizer. Given $t_0$ and $x_0 \in \R^d$, the continuous-time dynamics are given by the ordinary differential equation (ODE)
\begin{equation}
    \dot{x}(t) = -\Acal({\theta}, x(t), t), \quad \text{with initial condition } x(t_0) = x_0.
    \label{eq:ode_general_form}
\end{equation}
The trajectory $x(t)$ aims to minimize a function $f(x)$, and $\Acal$ is typically related to $f(x)$. Applying the forward Euler discretization method to the ODE \eqref{eq:ode_general_form} with a time step $h>0$ yields the iterative algorithm
\begin{equation}
    x_{k+1} = x_k - h\Acal({\theta}, x_k, t_k),
    \label{eq:iter_form_general}
\end{equation}
where $x_k$ is the approximation of $x(t_k)$, and $t_k = t_0 + kh$. We emphasize that $h$ serves as the discretization step for the ODE. The ``effective step size'' of the optimization algorithm at iteration $k$ is $h$ times any scaling factors embedded within $\Acal(\theta, x_k, t_k)$. We provide two simple examples of \eqref{eq:iter_form_general} as follows. This model also captures more sophisticated algorithms, such as the gradient method with momentum and LSTM-based learnable optimizers, as illustrated in Appendix \ref{app:examples}.

\begin{figure}[htbp]
    \centering
    \begin{subfigure}[b]{0.49\textwidth}
        \centering
        \includegraphics[width=\linewidth]{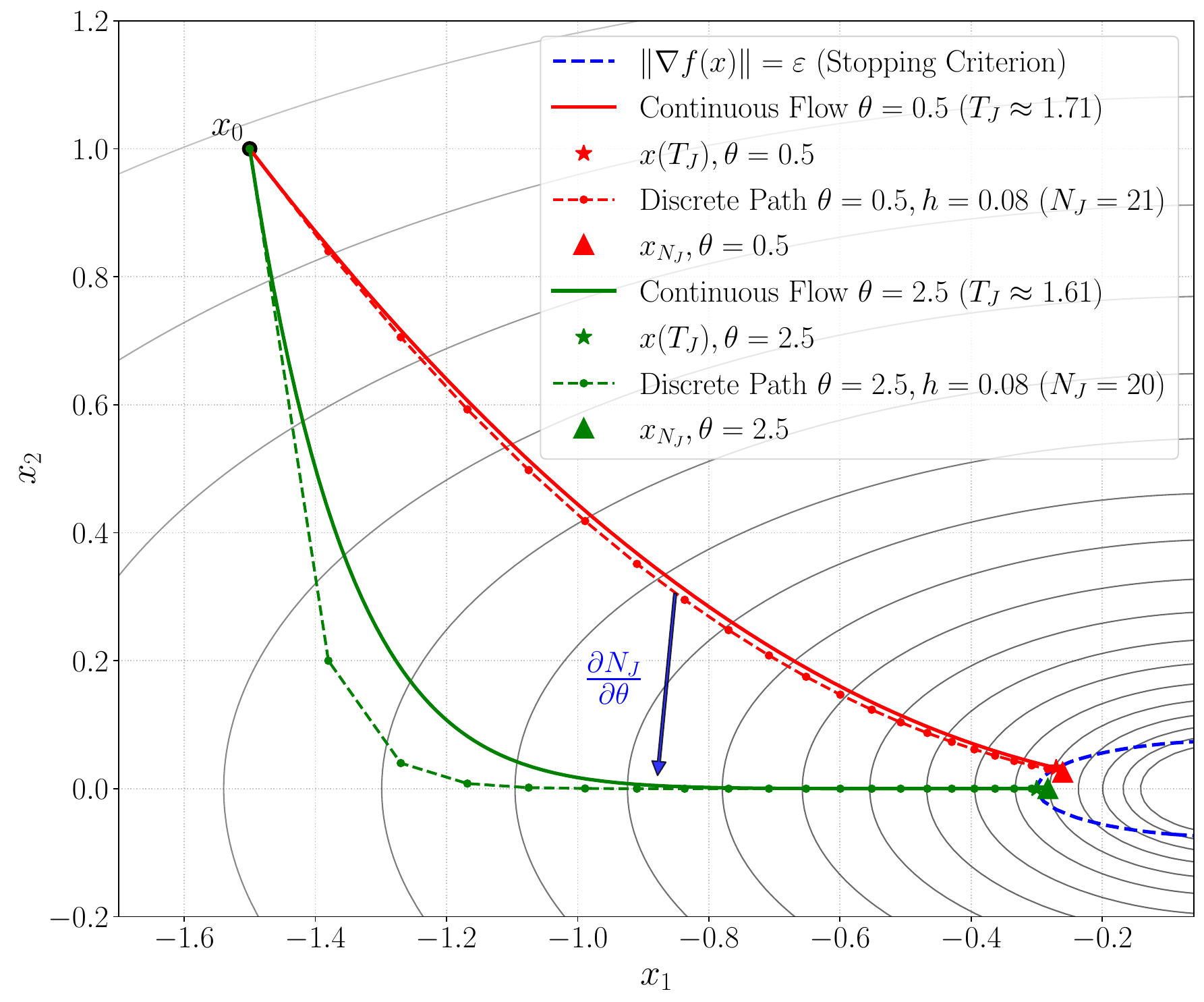}
        \caption{Trajectory vs discrete path}
        \label{fig:illustration}
    \end{subfigure}
    \hfill
    \begin{subfigure}[b]{0.49\textwidth}
        \centering
        \includegraphics[width=\linewidth]{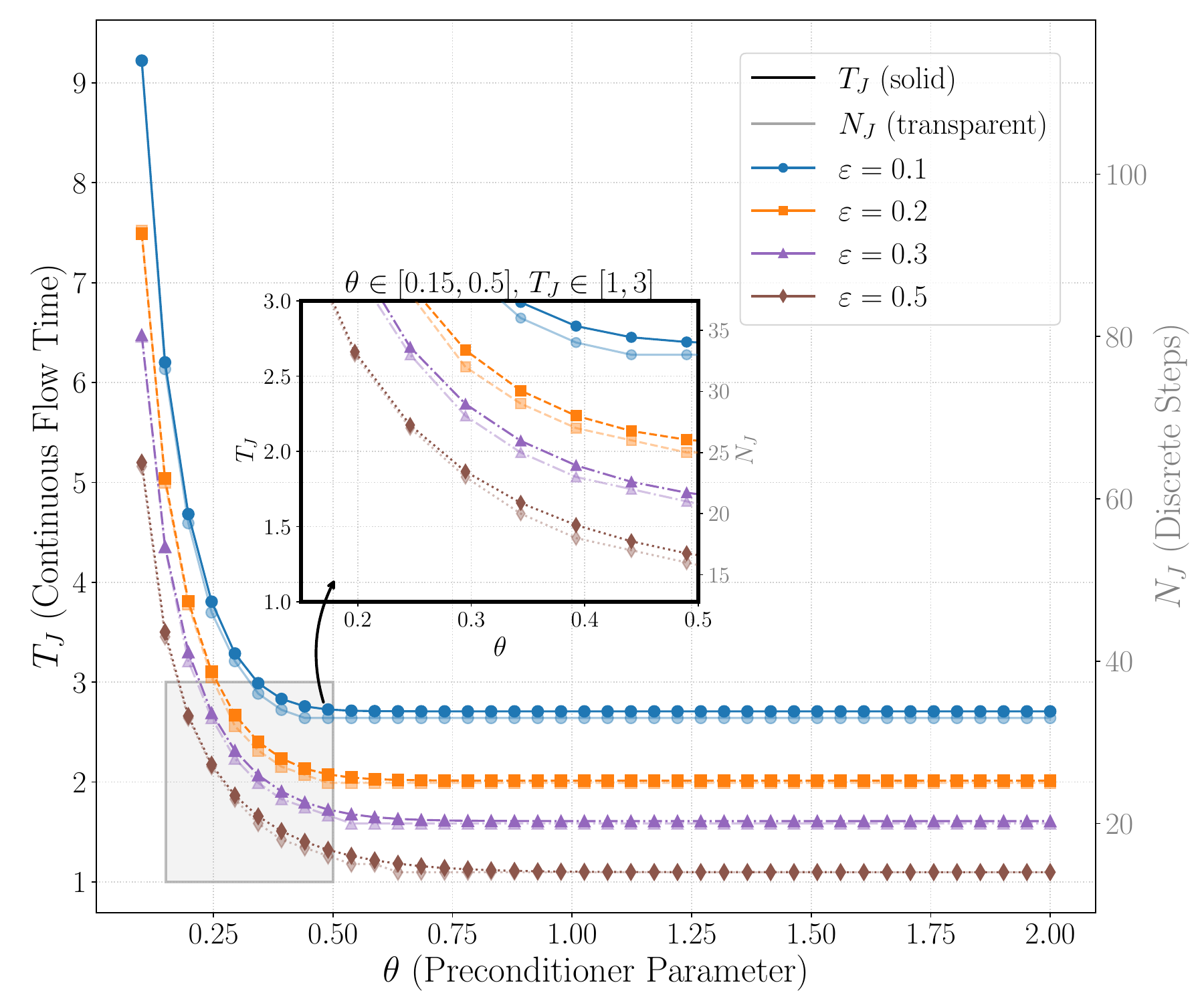}
        \caption{Stopping time vs parameters}
        \label{fig:time-theta}
    \end{subfigure}
    \caption{Illustration of the differentiable stopping time on $f(x_1,x_2)=0.5x_1^2+2x_2^2$ and $\mathcal{A}(x,\theta,t)=\operatorname{diag}(1,\theta)\nabla f(x)$. Effect of $\theta$ on continuous and discrete stopping time $T_J,N_J$ for different $\varepsilon$ values.}
    \label{fig:combined}
\end{figure}

\textbf{Rescaled Gradient Flow.}
A common instance is the rescaled gradient flow, where $\Acal$ incorporates a time-dependent and parameter-dependent scaling factor $\alpha(\theta, t)$ for the gradient of the objective function $f(x)$. In this case
\begin{equation}
    \Acal(\theta, x(t), t) = \alpha(\theta, t) \nabla f(x(t)).
\end{equation}
The ODE becomes $\dot{x}(t) = -\alpha(\theta, t) \nabla f(x(t))$. The parameters $\theta$ might define the functional form of $\alpha$, e.g., if $\alpha(\theta, t) = \theta_1 e^{-\theta_2 (t-t_0)}$, then $\theta=(\theta_1, \theta_2)$. The effective step size for the discretized iteration $x_{k+1} = x_k - h \alpha(\theta, t_k) \nabla f(x_k)$ is $h \cdot \alpha(\theta, t_k)$.

\textbf{Learned Optimizer.}
Another relevant scenario involves the optimizer using a parametric model, such as a neural network. Let $\mathcal{N}(\cdot; \theta)\colon \R^d\times\R^d\times\R\to\R^d$ denote a neural network parameterized by weights $\theta$. This network could learn, for instance, a diagonal preconditioning matrix $\operatorname{diag}(\mathcal{N}(x, \nabla f(x), t; \theta))$. Then, the function $\Acal$ is defined as
\begin{equation}
\label{eq:preconditioner}
    \Acal(\theta, x(t), t) = \operatorname{diag}(\mathcal{N}(x(t), \nabla f(x(t)), t; \theta)) \nabla f(x(t)).
\end{equation}
The discretized update would then use this learned preconditioned gradient.

\subsection{Differentiating the Continuous Stopping Time}
We first give a formal definition of the continuous stopping time.
\begin{definition}[Continuous Stopping Time]
Given a function $J$, for a stopping criterion defined by the condition $J(x)=\varepsilon$, the continuous stopping time is the first time that the trajectory reaches it:
\begin{equation}
    T_{J}(\theta,x_0,\varepsilon)\colon=\inf_{t}\{t\mid J(x(t))\leq \varepsilon ,  t\geq t_0, x(t) \text{ solves \eqref{eq:ode_general_form}}\}.
\end{equation}
When $J(x(t))$ never reaches the target $\varepsilon$, we set $T_{J}(\theta,x_0,\varepsilon)=+\infty$.
\end{definition}
Now, we present a theorem that establishes conditions for the differentiability of the stopping time $T_J$ with respect to parameters $\theta$ or the initial condition $x_0$.

\begin{theorem}[Differentiability of Continuous Stopping Time]
\label{thm:differentiability}
Let $T = T_J(\theta, x_0, \varepsilon)$ be the continuous stopping time such that $J(x(T)) = \varepsilon$.
We assume that the function $\Acal(\theta, x, t)$ is continuously differentiable with respect to $\theta$, $x$, and $t$. Additionally, the stopping criterion function $J(x)$ is assumed to be continuously differentiable with respect to $x$. Finally, it is assumed that the time derivative of the criterion function along the trajectory does not vanish at time $T$, that is,
$$
\frac{\dd}{\dd t} J(x(t)) \Big|_{t=T} = \nabla J(x(T))^\top \dot{x}(T) \neq 0.
$$
Then, the solution $x(t; \theta, x_0)$ of the ODE system is continuously differentiable with respect to its arguments $\theta$ and $x_0$ for $t$ in a neighborhood of $T$. The stopping time $T_J(\theta, x_0, \varepsilon)$ is continuously differentiable with respect to $\theta$ (and $x_0$) in a neighborhood of the given $(\theta, x_0)$ where $T_J < \infty$.
Specifically, its derivatives with respect to a component $\theta$ and $x_{0}$ are given by
\begin{equation}
    \frac{\partial T_J}{\partial \theta} = - \frac{\nabla J(x(T))^\top {\partial x(T)}/{\partial \theta}}{\nabla J(x(T))^\top \dot{x}(T)},\qquad
    \frac{\partial T_J}{\partial x_{0}} = - \frac{\nabla J(x(T))^\top {\partial x(T)}/{\partial x_{0}}}{\nabla J(x(T))^\top \dot{x}(T)}.
\end{equation}
\end{theorem}
The differentiability of $x(T)$ with respect $\theta$ and $x_0$ is guaranteed by the smooth dependence of solutions on initial conditions and parameters. The terms ${\partial x(T)}/{\partial \theta}$ and ${\partial x(T)}/{\partial x_{0}}$ are sensitivities of the state $x$ at time $T$ with respect to $\theta$ and $x_{0}$ respectively, which can be obtained by solving the corresponding sensitivity equations or via adjoint methods. The proof is an application of the implicit function theorem, which is deferred to Appendix \ref{app:differentiability}. For the differentiability under weaker conditions, one may refer to \cite[Proposition 2]{xie2024ode}, which confirms the path differentiability \cite{bolte2021conservative} when $\Acal$ involves non-smooth components.
% Condition 3 ensures that the trajectory $J(x(t))$ crosses the level $\varepsilon$ transversally, not tangentially. 

\subsection{Differentiable Discrete Stopping Time: An Effective Approximation}
The continuous stopping time $T_J$ is an ideal measure of algorithm efficiency. However, backpropagating through it requires solving forward and backward (adjoint) differential equations numerically. This can incur significant computational overhead, and many of the detailed steps evaluated by an ODE solver might be considered "wasted" compared to the coarser steps of the original iterative algorithm \eqref{eq:iter_form_general}. We now return to the discrete iteration \eqref{eq:iter_form_general} and propose an efficient approach to approximate $\nabla_{\theta} T_J$ and $\nabla_{x_0} T_J$.
\begin{definition}[Discrete Stopping Time]
\label{def:discrete-stopping-time}
For a stopping criterion $J(x)=\varepsilon$ and the iterative algorithm \eqref{eq:iter_form_general}, the discrete stopping time is the smallest integer such that $J(x_K) \le \varepsilon$:
\begin{equation}
    N_J(\theta, x_0, \varepsilon)\colon=\min_{n}\{n\mid J(x_n) \le \varepsilon,n\geq 0,\{x_n\}_{n=0}^\infty \text{ satisfies \eqref{eq:iter_form_general}} \}.
\end{equation}
If $J(x_n) > \varepsilon$ for all $n \ge 0$, we set $N_J = +\infty$.
\end{definition}
While $N_J$ is inherently an integer, to enable its use in gradient-based optimization of $\theta$ or $x_0$, we seek a meaningful way to define its sensitivity to these parameters. Our approach is inspired by Theorem \ref{thm:differentiability}. We conceptualize $N$ as a continuous variable for which the condition $J(x_N(\theta, x_0)) \approx \varepsilon$ holds exactly at the stopping time. Formally differentiating this identity with respect to $\theta$ gives
$$
0\approx\nabla J(x_N)^\top \left( \frac{\partial x_N}{\partial N} \textcolor{brown}{\frac{\partial N}{\partial \theta}} + \frac{\partial x_N}{\partial \theta} \right)\approx \frac{J(x_N)-J(x_{N-1})}{h}\textcolor{brown}{\frac{\partial N}{\partial \theta}} + \nabla J(x_N)^\top\frac{\partial x_N}{\partial \theta} .
$$
Under suitable regularity assumptions, this approximation allows us to define the sensitivity of the discrete stopping time.

\begin{definition}[Sensitivity of the Discrete Stopping Time]
\label{def:sensitivity_N}
Assume that the conditions of Theorem~\ref{thm:differentiability} hold. Let $N = N_J(\theta, x_0, \varepsilon)$ denote the discrete stopping time. Then, the sensitivities of $N$ with respect to $\theta$ and $x_{0}$ are defined as
\begin{equation}
\label{eq:def-diff-stop-time}
\frac{\partial N}{\partial \theta} \colon= - \frac{h\nabla J(x_N)^\top {\partial x_N}/{\partial \theta}}{J(x_N) - J(x_{N-1})},\qquad
\frac{\partial N}{\partial x_{0}} \colon= - \frac{h\nabla J(x_N)^\top {\partial x_N}/{\partial x_{0}}}{J(x_N) - J(x_{N-1})}.
\end{equation}
\end{definition}

Since Definition \ref{def:discrete-stopping-time} ensures $J(x_N) - J(x_{N-1})<0$, the above expressions are well-defined. Beyond being a natural symbolic differentiation of the discrete stopping condition, this definition also serves as an effective approximation of the gradient of the continuous stopping time. The next theorem formalizes this connection by quantifying the approximation error between the sensitivities of the discrete stopping time and the gradients of the continuous stopping time.

\begin{theorem}[Approximation Error for Gradient of Stopping Time]
\label{thm:approx-error}
% Assume further that the trajectory is discretized using the forward Euler method~\eqref{eq:iter_form_general} with step size $h$.
Let $J(x) = \varepsilon$ be a stopping criterion, and let $h > 0$ be a time step size. Assume the discrete stopping index $N_J$ satisfies $T_J \in (t_0 + (N_J - 1)h, \; t_0 + N_J h]$, where $T_J$ is the (continuous) stopping time and $t_0$ is the initial time. Suppose that the function $\mathcal{A}$ is twice continuously differentiable. We assume that $\mathcal{A}(\theta,x(t),t)$, regarded as a function of $(\theta, t)$, has uniformly bounded $W^{2,\infty}$ norms with respect to $(\theta, t)$ in a neighborhood of $\theta \times [t_0,\; t_0 + N_J h]$, and that $J$, regarded as a function of $x$, has a uniformly bounded $W^{2,\infty}$ norm in a neighborhood of $x(T_J)$. Furthermore, suppose the boundary condition $\nabla J(x(T_J))^\top\dot{x}(T_J)\neq 0$ holds. Then, for sufficiently small $h$, the following holds
\begin{equation}
\label{eq:error-bound}
\left\| \nabla_{\theta} T_J(\theta, x_0, \varepsilon) - \nabla_{\theta} N_J(\theta, x_0, \varepsilon) \right\| = \mathcal{O}(h).
\end{equation}
\end{theorem}
This theorem demonstrates that Definition~\ref{def:sensitivity_N} provides an approximation to the gradient of the continuous stopping time. That is, $\nabla_{\theta} N_J$ converges to $\nabla_{\theta} T_J$ as $h \to 0$. An analogous result holds for the gradient with respect to $x_0$. This result serves as a theoretical justification for using the symbolic discrete sensitivity \eqref{eq:def-diff-stop-time} as a surrogate for the continuous counterpart. In \cite{xie2024ode}, it is proved that, under mild condition, given a special form of $\Acal$, we have $\|x_{k}-x(t_k)\|\to 0$ using forward Euler discretization \eqref{eq:iter_form_general} with a fixed time step size $h$. This sheds the light for future improvement of the estimate in \eqref{eq:error-bound} without assuming sufficiently small time step size.

\subsection{Efficient Computation of the Sensitivity}
The primary challenge in \eqref{eq:def-diff-stop-time} is computing the numerator term $\nabla J(x_N)^\top \partial x_N / \partial \theta$. If the function $\Acal$ and the iteration process are implemented within an automatic differentiation framework (e.g., PyTorch, TensorFlow) where $\Acal$ might be a learnable \texttt{nn.Module}, then the numerator can be obtained by unrolling the computation graph and applying backpropagation. However, this can be unstable and memory-intensive for large $N$. Other methods include finite differences or stochastic gradient estimators, which are inexact.

Alternatively, the discrete adjoint method provides a memory-efficient way to compute the required vector-Jacobian products $\nabla J(x_N)^\top (\partial x_N / \partial \theta)$ and $\nabla J(x_N)^\top (\partial x_N / \partial x_0)$ without forming the Jacobians explicitly. This method involves a forward pass to compute the trajectory $x_0, \dots, x_N$, followed by a backward pass that propagates adjoint (or co-state) vectors. Let $x_{k+1} = G_k(x_k, \theta) = x_k - h\Acal(\theta, x_k, t_k)$ be the iterative update. Algorithm~\ref{alg:discrete_adjoint_computation} outlines the procedure to compute the term $S_{\theta} = \nabla J(x_N)^\top (\partial x_N / \partial \theta)$ and $S_{x_0} = \nabla J(x_N)^\top (\partial x_N / \partial x_0)$. The correctness of Algorithm~\ref{alg:discrete_adjoint_computation} is established by the following theorem. The proof is presented in Appendix \ref{app:proof-adjoint}.

\begin{algorithm}
\caption{Discrete Adjoint Method for Sensitivity Components}
\label{alg:discrete_adjoint_computation}
\begin{algorithmic}[1]
\State \textbf{Input:} Forward trajectory $\{x_k\}_{k=0}^N$, parameters $\theta$, $J(x_N)$, time step $h$, initial time $t_0$.
\State \textbf{Output:} $S_{\theta} = \nabla J(x_N)^\top (\partial x_N / \partial \theta)$ and $S_{x_0} = \nabla J(x_N)^\top (\partial x_N / \partial x_0)$.

% \State \textit{// Initialize adjoint vector}
\State $\lambda \gets \nabla J(x_N)$.  \Comment{Initialize adjoint vector}
% \State \textit{// Initialize sensitivity component for $\theta$}
\State $S_{\theta} \gets \bm{0}$ (vector of same size as $\theta$). \Comment{Initialize sensitivity component for $\theta$}

% \State \textit{// Backward propagation of adjoints}
\For{$k = N-1$ \textbf{downto} $0$}
    \State $t_k \gets t_0 + kh$.
    % \State \textit{// Accumulate contribution to $S_{\theta}$ from current step's explicit dependency on $\theta$}
    % \State $\frac{\partial \Acal_k}{\partial \theta} \gets \frac{\partial \Acal(\theta, x_k, t_k)}{\partial \theta}$.
    \State $S_{\theta} \gets S_{\theta} - h \left(\frac{\partial \Acal(\theta, x_k, t_k)}{\partial \theta}\right)^\top \lambda$. \Comment{Accumulate contribution to $S_{\theta}$}
    % \State \textit{// Propagate adjoint vector backward through the $k$-th step dynamics}
    % \State $\frac{\partial \Acal_k}{\partial x_k} \gets \frac{\partial \Acal(\theta, x_k, t_k)}{\partial x_k}$. \Comment{Propagate adjoint vector backward}
    \State $\lambda \gets \left(I - h \frac{\partial \Acal(\theta, x_k, t_k)}{\partial x_k}\right)^\top \lambda$. \Comment{Propagate adjoint vector backward}
\EndFor
\State $S_{x_0} \gets \lambda$. \Comment{After the loop, $\lambda$ represents $\nabla J(x_N)^\top (\partial x_N / \partial x_0)$}
\State \Return $S_{\theta}$, $S_{x_0}$.
\end{algorithmic}
\end{algorithm}

\begin{proposition}[Discrete Adjoint Method]
\label{thm:adjoint_correctness}
Let the sequence $x_0, \dots, x_N$ be generated by $x_{k+1} = x_k - h\Acal(\theta, x_k, t_k)$.
The quantities $S_{\theta}$ and $S_{x_0}$ computed by Algorithm~\ref{alg:discrete_adjoint_computation} are equal to $\nabla J(x_N)^\top (\partial x_N / \partial \theta)$ and $\nabla J(x_N)^\top (\partial x_N / \partial x_0)$, respectively.
\end{proposition}

Once $S_{\theta}$ and $S_{x_0}$ are computed using Algorithm \ref{alg:discrete_adjoint_computation}, they are plugged into expression \eqref{eq:def-diff-stop-time}. This approach computes the required numerators efficiently by only requiring storage for the forward trajectory $\{x_k\}$ and the current adjoint vector $\lambda$, making its memory footprint $O(Nd + d)$, which is typically much smaller than $O(N \times \text{memory for } \Acal \text{ graph})$ needed for naive unrolling. The computational cost is roughly proportional to $N$ times the cost of evaluating $\Acal$ and its relevant partial derivatives (or VJPs).
The overall procedure to compute $\nabla_{\theta} N_J$ would first run the forward pass to find $N$ and store $\{x_k\}$, then call Algorithm~\ref{alg:discrete_adjoint_computation} to get $S_{\theta}$, and finally assemble the components using \eqref{eq:def-diff-stop-time}.

\section{Applications of Differentiable Stopping Time}
\label{sec:applications}
In this section, we explore two applications of the differentiable discrete stopping time: L2O and online adaptation of learning rates (or other optimizer parameters). The ability to differentiate $N_J$ allows us to directly optimize for algorithmic efficiency towards target suboptimality.

\subsection{L2O with Differentiable Stopping Time}
\label{sec:l2o_application}
In L2O, the objective is to learn an optimization algorithm, denoted by \eqref{eq:iter_form_general}, parameterized by $\theta$, that performs efficiently across a distribution of optimization tasks. Traditional L2O approaches often aim to minimize a sum of objective function values over a predetermined number of steps. While this provides a dense reward signal, it may not directly optimize for the speed to reach a specific target precision $\varepsilon$. To overcome this limitation, the L2O training objective can be augmented with the stopping time
\begin{equation}
\label{eq:l2o_objective}
    \min_{\theta} \quad \mathcal{L}(\theta) = \E_{f \sim \mathcal{D}_f, x_0 \sim \mathcal{D}_{x_0}} \left[ \sum_{k=0}^{K_{\text{max}}} w_k f(x_k) + \lambda N_J(\theta, x_0, \varepsilon) \right],
\end{equation}
where $\mathcal{D}_f,\mathcal{D}_{x_0}$ are distributions of $f$ and $x_0$, respectively, $K_{\text{max}}$ is a maximum horizon for the sum, $J$ is a stopping criterion depending on $f$, $w_k$ are weights, $\lambda$ is a balancing hyperparameter, and $N_J(\theta, x_0, \varepsilon)$ is the discrete stopping time. The parameters $\theta$ are then updated using a stochastic optimization method such as stochastic gradient descent or Adam. The update follows the rule
\begin{equation}
\theta_{\text{new}} = \theta_{\text{old}} - \eta_{\text{L2O}} \left( \nabla_{\theta} \left[\sum_{k=0}^{K_{\text{max}}} w_k f(x_k)\right] + \lambda \nabla_{\theta} N_J(\theta, x_0, \varepsilon) \right),
\end{equation}
where $\eta_{\text{L2O}}$ is the meta-learning rate. Combining these two losses contributions provides a richer training signal that values both the quality of the optimization path and the overall convergence speed.

Suppose $f(x_k)>f(x_{k+1})$ holds for all $k$, another interesting result comes from the identity
% \begin{equation}
% \label{eq:identity}
% \begin{aligned}
% \frac{\mathrm{d}}{\mathrm{d} \theta} \sum_{k=0}^{K_{\max}} f(x_{k}) =&\sum_{k=0}^{K_{\max}}(f(x_{k})-f(x_{k-1}))\frac{\nabla f(x_{k})\partial x_{k}/\partial \theta}{f(x_{k})-f(x_{k-1})}\\
% =&|f(x_{k-1})-f(x_{k})|\frac{\partial}{\partial\theta}N_{f}(\theta,x_{k-1},f(x_{k}))
% =\frac{\partial}{\partial\theta}\sum_{k=0}^{K_{\max}}(f(x_{k-1})-f(x_{k}))N_{f}(\theta,x_{k-1},f(x_{k}))
% \end{aligned}
% \end{equation}
\begin{equation}
\label{eq:identity}
\begin{aligned}
\frac{\mathrm{d}}{\mathrm{d} \theta} \sum_{k=0}^{K_{\max}} f(x_{k}) =&\sum_{k=0}^{K_{\max}}(f(x_{k})-f(x_{k-1}))\frac{\nabla f(x_{k})\partial x_{k}/\partial \theta}{f(x_{k})-f(x_{k-1})}\\
=&\frac{\partial}{\partial\theta}\sum_{k=0}^{K_{\max}}(f(x_{k-1})-f(x_{k}))N_{f}(\theta,x_{k-1},f(x_{k})).\\
\end{aligned}
\end{equation}
We emphasize that the operator $\partial /\partial \theta$ directly applies to the variable $\theta$ while ${\mathrm{d}}/{\mathrm{d} \theta}$ will unroll the intermediate variable and apply chain rule. The identity \eqref{eq:identity} reveals that optimizing the weighted loss sum with $w_k\equiv 1$ equals to minimize the sum of stopping times greedily with stopping criterion $f$ and natural weights $f(x_{k-1})-f(x_{k})$.

\subsection{Online Adaptation of Optimizer Parameters via Stopping Time}
\label{sec:online_adaptation_stopping_time_general}
Online adaptation of optimizer hyperparameters $\theta_k$ (for $x_{k+1} = G(x_k, \theta_k)$) can be triggered by an adaptive criterion $\varphi(N, \varepsilon)$. This criterion, potentially adaptive itself, signals when to update $\theta_k$. $N$ is a stopping time from a reference $x_{\text{ref}}$ (last adaptation point or $x_0$) until $\varphi$ is met at $x_{\text{current}}$. Upon meeting $\varphi$ at $x_{k+1} (=x_{\text{current}})$, the sensitivity $\partial N / \partial \theta$ of the stopping time $N$ to hyperparameters $\theta$ active within $[x_{\text{ref}}, x_{k+1}]$ is key. Theoretically, $\partial N / \partial \theta$ is found by backpropagating through all steps from $x_{k+1}$ to $x_{\text{ref}}$, yielding a principled multi-step signal for adjusting $\theta$.

\begin{figure}[H]
    \centering
    \begin{minipage}[t]{0.5\textwidth} % Adjusted width
        \vspace{0pt} % For alignment
         Calculating the full multi-step $\partial N / \partial \theta$ to $x_{\text{ref}}$ is often costly. Practical methods may truncate this dependency. The simplest truncation considers only the immediate impact of $\theta_k$ on $x_{k+1}$. For this single-step proxy $N_k$, its sensitivity, given $x_{k+1} = x_k - h_{\text{step}}\mathcal{A}(\theta_k, x_k, t_k)$ and decreasing $J(x)$, is
        % \begin{equation}
        %     \frac{\partial N_k}{\partial \theta} = - \frac{\nabla J(x_{k+1})^\top (\partial x_{k+1} / \partial \theta)}{J(x_{k+1}) - J(x_k)},
        %     \label{eq:local_sensitivity_N_theta_general}
        % \end{equation}
        % and with $\partial x_{k+1} / \partial \theta = -h_{\text{step}} \partial \mathcal{A}(\theta_k, x_k, t_k) / \partial \theta$:
        \begin{equation}
            \frac{\partial N_k}{\partial \theta} = \frac{h_{\text{step}} \nabla J(x_{k+1})^\top (\partial \mathcal{A}(\theta_k, x_k, t_k) / \partial \theta)}{J(x_{k+1}) - J(x_k)}.
            \label{eq:dN_dtheta_general_A}
        \end{equation}
        For Adam's learning rate $\alpha_k$ (where $x_{k+1} = x_k - \alpha_k d_k$, $\mathcal{A} = \alpha_k d_k$, $h_{\text{step}}=1$, $\theta_k=\alpha_k$, and $\partial \mathcal{A} / \partial \alpha_k = d_k$), the one-step truncated sensitivity $S_k$ from \eqref{eq:dN_dtheta_general_A} (with $J(x)=f(x)$) becomes
        \begin{equation}
            S_k = \frac{\nabla f(x_{k+1})^\top d_k}{f(x_{k+1}) - f(x_k)}.
            \label{eq:Sk_definition_code_aligned}
        \end{equation}
        $S_k$ is the practical signal for adjusting $\alpha_k$. If $f(x_{k+1}) < f(x_k)$ (negative denominator), $\alpha_k$ is updated by
        \begin{equation}
            \alpha_{k+1} = \alpha_k - \eta_{\text{adapt}} S_k,
            \label{eq:alpha_update_code_aligned}
        \end{equation}
        with $\eta_{\text{adapt}}$ as the adaptation rate. Specifically, if $\nabla f(x_{k+1})^\top d_k > 0$, then $S_k < 0$, increasing $\alpha_k$; if $\nabla f(x_{k+1})^\top d_k < 0$, then $S_k > 0$, decreasing $\alpha_k$. This behavior's empirical efficacy requires validation. Algorithm \ref{alg:adam_ola} presents Adam with Online LR Adaptation (Adam-OLA).
    \end{minipage}
    \hfill % Ensures separation
    \begin{minipage}[t]{0.49\textwidth} % Adjusted width
        \vspace{-10pt} % For alignment
        \begin{algorithm}[H]
            \caption{Adam-OLA}
            \label{alg:adam_ola}
            \begin{algorithmic}[1]
                \State \textbf{Input:} $x_0$, $\alpha_0$, $f$, $\nabla f$.
                \State \textbf{Params:}$\beta_1, \beta_2, \varepsilon_{\text{stab}},\eta_{\text{adapt}}, \epsilon_{\text{desc}}$.
                \State $m_0, v_0 \gets 0, 0$; $k \gets 0$; $\alpha_{\text{curr}} \gets \alpha_0$.
                \State $x_{\text{ref}} \gets x_0$; $f_{\text{ref}} \gets f(x_0)$; $N_{\text{updates}} \gets 0$.
                \For{$k = 0, 1, \dots$ until convergence}
                    \State $g_k \gets \nabla f(x_k)$
                    \State $m_{k+1} \gets \beta_1 m_k + (1-\beta_1) g_k$
                    \State $v_{k+1} \gets \beta_2 v_k + (1-\beta_2) g_k^2$
                    \State $\hat{m}_{k+1} \gets m_{k+1}/(1-\beta_1^{k+1})$
                    \State $\hat{v}_{k+1} \gets v_{k+1}/(1-\beta_2^{k+1})$
                    \State $d_k \gets \hat{m}_{k+1} / (\sqrt{\hat{v}_{k+1}} + \varepsilon_{\text{stab}})$
                    \State $f_k^{\text{prev}} \gets f(x_k)$
                    \State $x_{k+1} \gets x_k - \alpha_{\text{curr}} d_k$
                    \State $f_{k+1} \gets f(x_{k+1})$
                    \If{$f_{\text{ref}} - f_{k+1} > \epsilon_{\text{desc}} \cdot N_{\text{updates}}$} \label{algoline:adapt_cond_varphi_concise}
                    % \Comment{$\varphi$ met}
                        \State $g_{k+1}^{\text{new}} \gets \nabla f(x_{k+1})$
                        \State $\Delta f_{\text{step}} \gets f_{k+1} - f_k^{\text{prev}}$
                        \If{$\Delta f_{\text{step}} \neq 0$}
                            \State $S_k \gets (g_{k+1}^{\text{new}} \cdot d_k) / \Delta f_{\text{step}}$
                            \State $\alpha_{\text{curr}} \gets \alpha_{\text{curr}} - \eta_{\text{adapt}} S_k$
                        \EndIf
                        \State $x_{\text{ref}} \gets x_{k+1}$; $f_{\text{ref}} \gets f_{k+1}$
                        \State $N_{\text{updates}} \gets N_{\text{updates}} + 1$
                    \EndIf
                \EndFor
                \State \textbf{Return} $x_k$
            \end{algorithmic}
        \end{algorithm}
    \end{minipage}
\end{figure}
% $\alpha_{\text{curr}}$ adaptation is triggered when its specific criterion $f_{\text{ref}} - f_{k+1} > \epsilon_{\text{desc}} \cdot N_{\text{updates}}$ (an instance of $\varphi(N, \varepsilon)$) is met. $f_{\text{ref}}=f(x_{\text{ref}})$ is the loss at the start of the adaptation window (last LR update or $f(x_0)$), and $N_{\text{updates}}$ counts LR adaptation events.
% This online adaptation enhances Adam's robustness by dynamically adjusting its learning rate based on the one-step truncated sensitivity $S_k$. Note that while Algorithm \ref{alg:adam_ola} uses this single-step $S_k$ (Eq. \eqref{eq:Sk_definition_code_aligned}), the theoretical stopping time sensitivity $\partial N / \partial \theta$ implies backpropagation from $x_{k+1}$ to $x_{\text{ref}}$. For online tractability, this full backpropagation is truncated; Adam-OLA uses the simplest one-step version. Propagating over more steps could capture richer dynamics but incurs higher computational costs.

\section{Experiments}
\label{sec:experiments}

\textbf{Validation of Theorems \ref{thm:approx-error} and Proposition \ref{thm:adjoint_correctness}.}
To validate the effectiveness and efficiency of our differentiable discrete stopping time approach, we conduct experiments on a high-dimensional quadratic optimization problem. We minimize $f(x) = x^\top Q x/2, x\in\mathbb{R}^d$ with $d \in \{10^2, 10^3, 10^4\}$ and condition number 100. The optimization algorithm uses forward Euler discretization \eqref{eq:iter_form_general} of \eqref{eq:ode_general_form}, where $\mathcal{A}$ incorporates a diagonal preconditioner \eqref{eq:preconditioner} with $10d$ learnable parameters. The stopping criterion is $\|\nabla f(x)\|_2^2 \le \varepsilon$ with $\varepsilon \in \{10^{-3}, 10^{-4}, 10^{-5}\}$. We compare the sensitivity of the discrete stopping time $\nabla_{\theta} N_J$, computed using Algorithm \ref{alg:discrete_adjoint_computation}, against the gradient of the continuous stopping time $\nabla_{\theta} T_J$ (ground truth), computed via \texttt{torchdiffeq} \cite{chen2018neural} through an adaptive ODE solver. We vary $d$, $\varepsilon$, and $h$.
% The groundbreaking work introduced Neural ODEs as a continuous-depth alternative to discrete layer architectures, establishing the foundations for differentiable ODE-based neural networks .

We evaluate effectiveness and efficiency using two primary metrics, \emph{Relative Error} quantifies the accuracy of $\nabla_{\theta} N_J$ as an approximation of $\nabla_{\theta} T_J$. A smaller error indicates better accuracy, with $\mathcal{O}(h)$ magnitude expected (Theorem \ref{thm:approx-error}). Results are shown in Figure~\ref{fig:error}. \emph{NFE Ratio} measures the computational cost efficiency, defined as the number of function evaluations (NFE) for Algorithm \ref{alg:discrete_adjoint_computation} to compute $\nabla_{\theta} N_J$ versus the adaptive ODE solver to compute $\nabla_{\theta} T_J$. A ratio $<1$ indicates the discrete approach's forward simulation is cheaper. Results are shown in Figure~\ref{fig:ratio}. The numbers of Euler NFE and ODE NFE are presented in Appendix \ref{app:examples}. The math formulae of these quantities are
\[
\text{Relative Error} = \frac{\|\nabla_{\theta} N_J - \nabla_{\theta} T_J\|_2}{\|\nabla_{\theta} T_J\|_2 + \|\nabla_{\theta} N_J\|_2},\qquad
\text{NFE Ratio} = \frac{\text{Euler NFE}}{\text{ODE NFE}}.
\]
By analyzing the relative error and NFE ratio across varying parameters, our experiments demonstrate that the discrete sensitivity provides an accurate approximation while requiring substantially fewer function evaluations for the forward pass, highlighting its efficiency and suitability for high-dimensional problems compared to methods relying on precise ODE solves for the stopping time gradient.

\begin{figure}[htbp]
    \centering
    \begin{subfigure}[b]{0.49\textwidth}
        \centering
        \includegraphics[width=\linewidth]{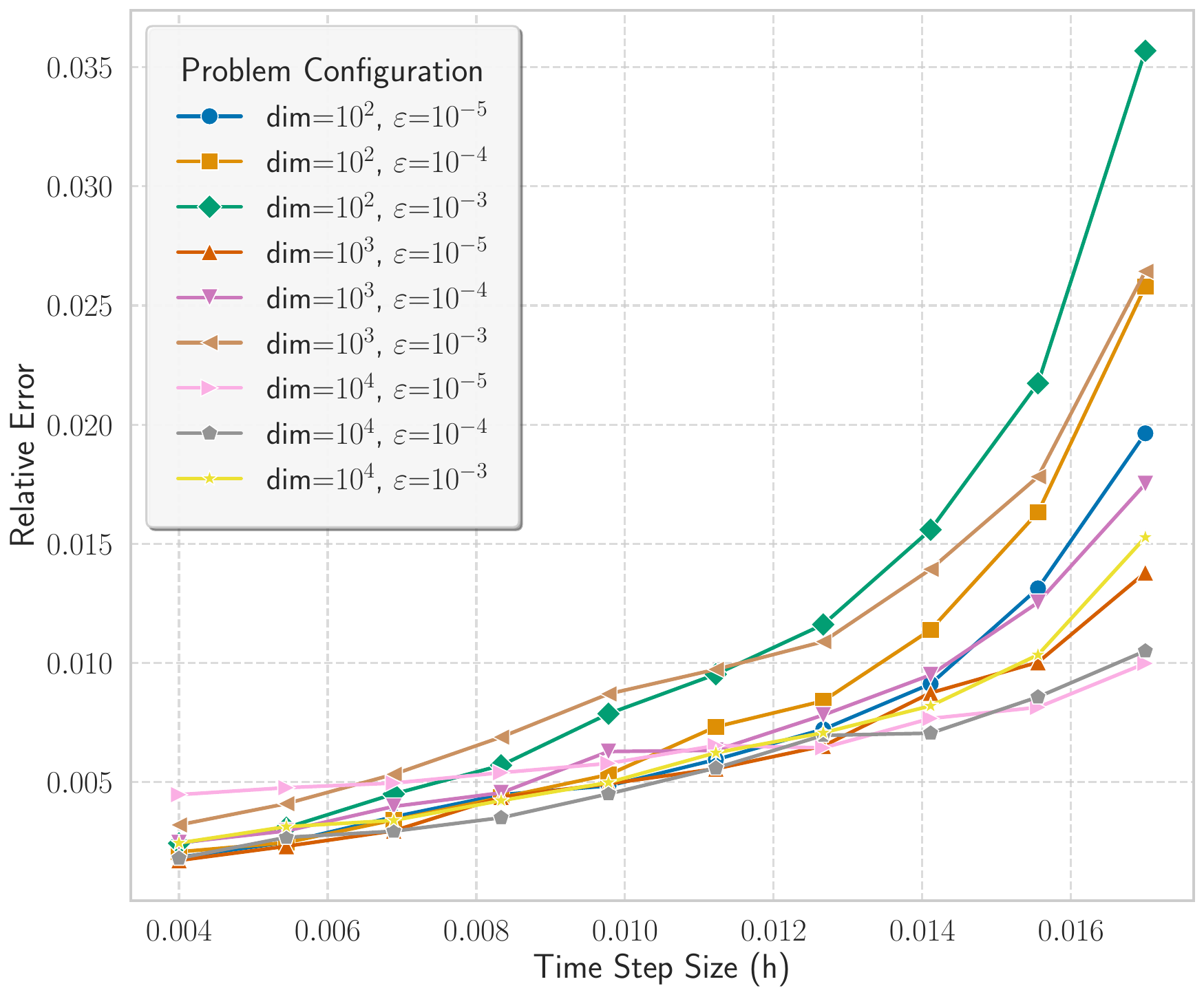}
        \caption{Relative Error vs Time Step Size}
        \label{fig:error}
    \end{subfigure}
    \hfill
    \begin{subfigure}[b]{0.49\textwidth}
        \centering
        \includegraphics[width=\linewidth]{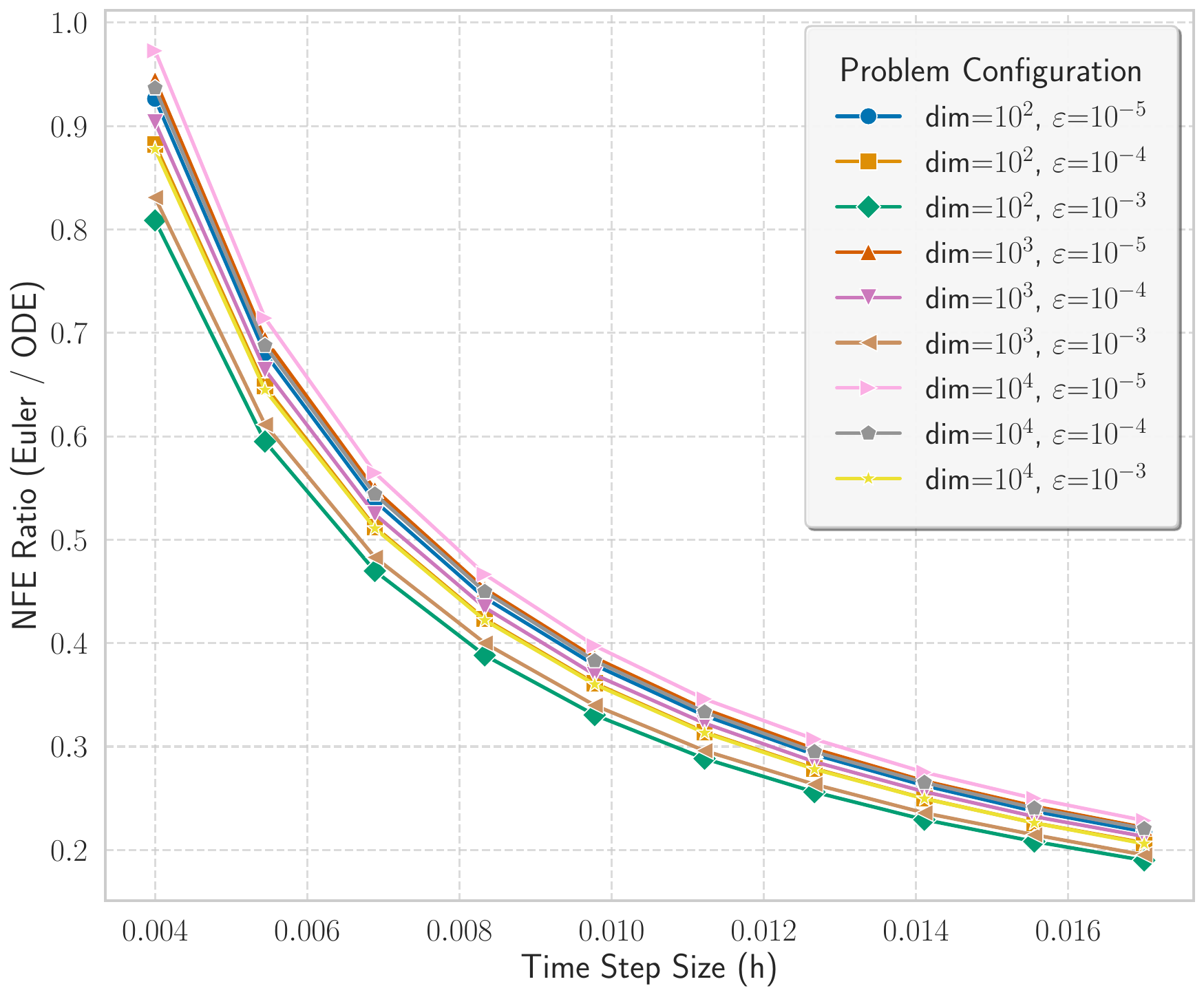}
        \caption{NFE Ratio vs Time Step Size}
        \label{fig:ratio}
    \end{subfigure}
    \caption{Experimental results comparing the discrete and continuous stopping time gradients across varying problem dimensions, stopping thresholds, and time step sizes. (a) shows the relative error of the discrete gradient approximation. (b) shows the computational cost ratio.}
    \label{fig:combined_results} % Added a combined caption and label for the figure environment
\end{figure}

\textbf{Learning to Optimize.}
We consider a logistic regression problem with synthetic data
$$
\min_{x \in \mathbb{R}^d} \;f(x)\colon= \frac{1}{n} \sum_{i=1}^n \log\left(1 + \exp(-y_i w_i^\top x)\right),
$$
where $w_i \in \mathbb{R}^d$ denotes the $i$-th data sample, and $y_i \in \{0,1\}$ is the corresponding label. We consider two L2O optimizers: L2O-DM~\cite{l2o-dm} and L2O-RNNprop~\cite{rnnprop}. Both employ a two-layer LSTM with a hidden state size of 30 to predict coordinate-wise updates. The data generation process and the architecture of the L2O optimizers are detailed in Appendix~\ref{app:examples}. The training setup follows that of~\cite{rnnprop}. Specifically, the feature dimension is set to $d = 512$, and the number of samples is $n = 256$. In each training step, we use a mini-batch consisting of 64 optimization problems. The total number of training steps is 500, and the optimizers are allowed to run for at most $K_{\max} = 100$ steps. We divide the sequence into 5 segments of 20 steps each and apply truncated backpropagation through time (BPTT) for training. The weights in \eqref{eq:l2o_objective} are set as $w_k \equiv 1 / K_{\max}$. Two loss functions are considered. The first corresponds to setting $\lambda = 0$ in \eqref{eq:l2o_objective}, resulting in an average loss across all iterations. To demonstrate the benefit of incorporating the stopping time penalty, we also set $\lambda = 1$ and use the stopping criterion $f(x_{k-1}) - f(x_k) \leq 10^{-5}$. This can be reformulated into the standard form by augmenting the state variable as $z_k = (x_k, x_{k-1})$ and defining $J(z) = f(z[d+1:2d]) - f(z[1:d])$.

The test results are summarized in Figure~\ref{fig:l2o}. \texttt{L2O-DM} refers to the L2O-DM optimizer. \texttt{L2O-RNNprop} and \texttt{L2O-RNNprop-Time} denote the L2O-RNNprop optimizer with and without the stopping time penalty, respectively. Since L2O-DM does not reach the stopping criterion within the maximum number of steps, we do not evaluate its performance under the stopping time penalty. For comparison with manually designed optimizers, \texttt{GD} represents gradient descent, \texttt{NAG} denotes Nesterov's accelerated gradient method, and \texttt{Adam} is a well-known adaptive optimizer. All classical methods use a fixed step size of $1/L$, where $L$ is the Lipschitz constant of $\nabla f(x)$ estimated at the initial point $x_0$. Our results show a clear acceleration toward reaching the target stopping criterion. In Figure~\ref{fig:logistic_syn}, we evaluate on a problem of the same size, $d = 512$, $n = 256$. In Figure~\ref{fig:logistic_syn_large}, we test on a fourfold larger instance with $d = 2048$, $n = 1024$. Both experiments indicate that the number of iterations required to meet the stopping criterion is reduced by hundreds of steps, and the learned optimizers generalize robustly to larger-scale problems.

\begin{figure}[htbp]
    \centering
    \begin{subfigure}[b]{0.49\textwidth}
        \centering
        \includegraphics[width=\linewidth]{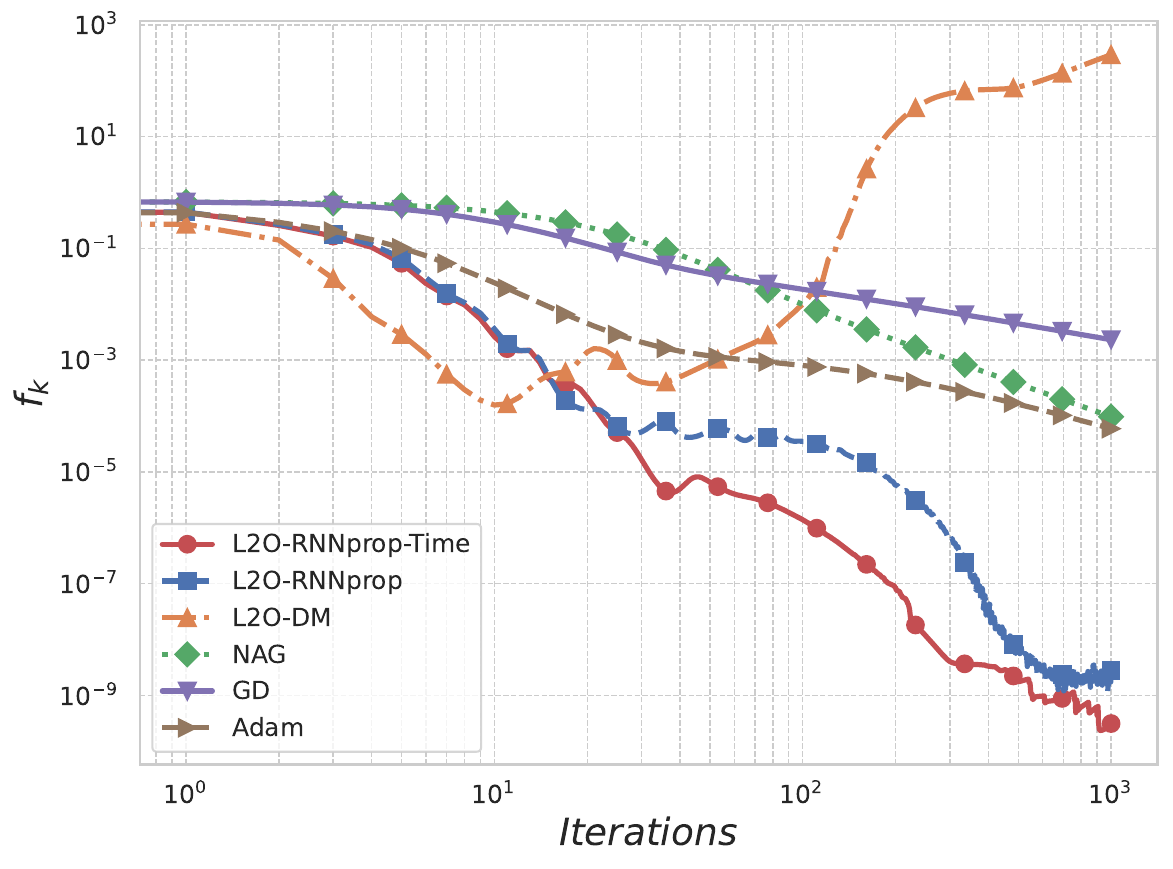}
        \caption{Train and test on the same size problems.}
        \label{fig:logistic_syn}
    \end{subfigure}
    \hfill
    \begin{subfigure}[b]{0.49\textwidth}
        \centering
        \includegraphics[width=\linewidth]{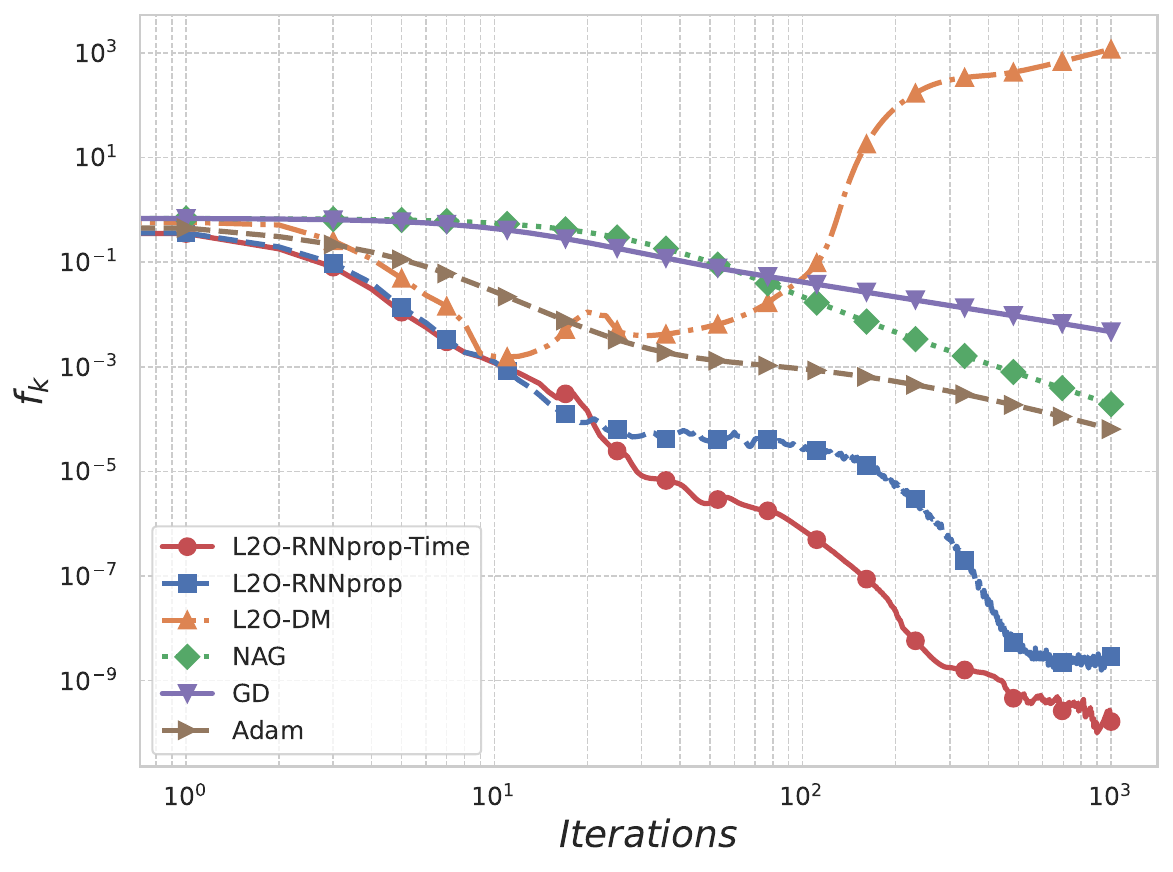}
        \caption{Test on 4x larger problems than training.}
        \label{fig:logistic_syn_large}
    \end{subfigure}
    \caption{Test results of different optimizers on logistic regression: Function value versus iteration.}
    \label{fig:l2o}
\end{figure}

\textbf{Online Learning Rate Adaptation.}
We tested Algorithm~\ref{alg:adam_ola} on smooth support vector machine (SVM) problems~\cite{ssvm}, using datasets from LIBSVM~\cite{libsvm}. \texttt{HB} denotes the heavy-ball method, and \texttt{NAG-SC} refers to the Nesterov accelerated gradient method tailored for strongly convex objectives. \texttt{Adagrad} is an adaptive gradient algorithm that scales the learning rate per coordinate based on historical gradient information. \texttt{Adam-HD} is an influential extension of Adam~\cite{atilim2018online} in the context of online learning rate adaptation; it updates the base learning rate of \texttt{Adam} at each iteration using a hyper-gradient technique. The remaining abbreviations retain their previously defined meanings. The results presented in Figure~\ref{fig:ola} demonstrate that Algorithm~\ref{alg:adam_ola} consistently outperforms the baseline methods, particularly in the later stages of convergence. Further comparisons across multiple datasets, as well as detailed descriptions of hyperparameter settings for the baselines, are provided in Appendix~\ref{app:examples}.

\begin{figure}[htbp]
    \centering
    \begin{subfigure}[b]{0.49\textwidth}
        \centering
        \includegraphics[width=\linewidth]{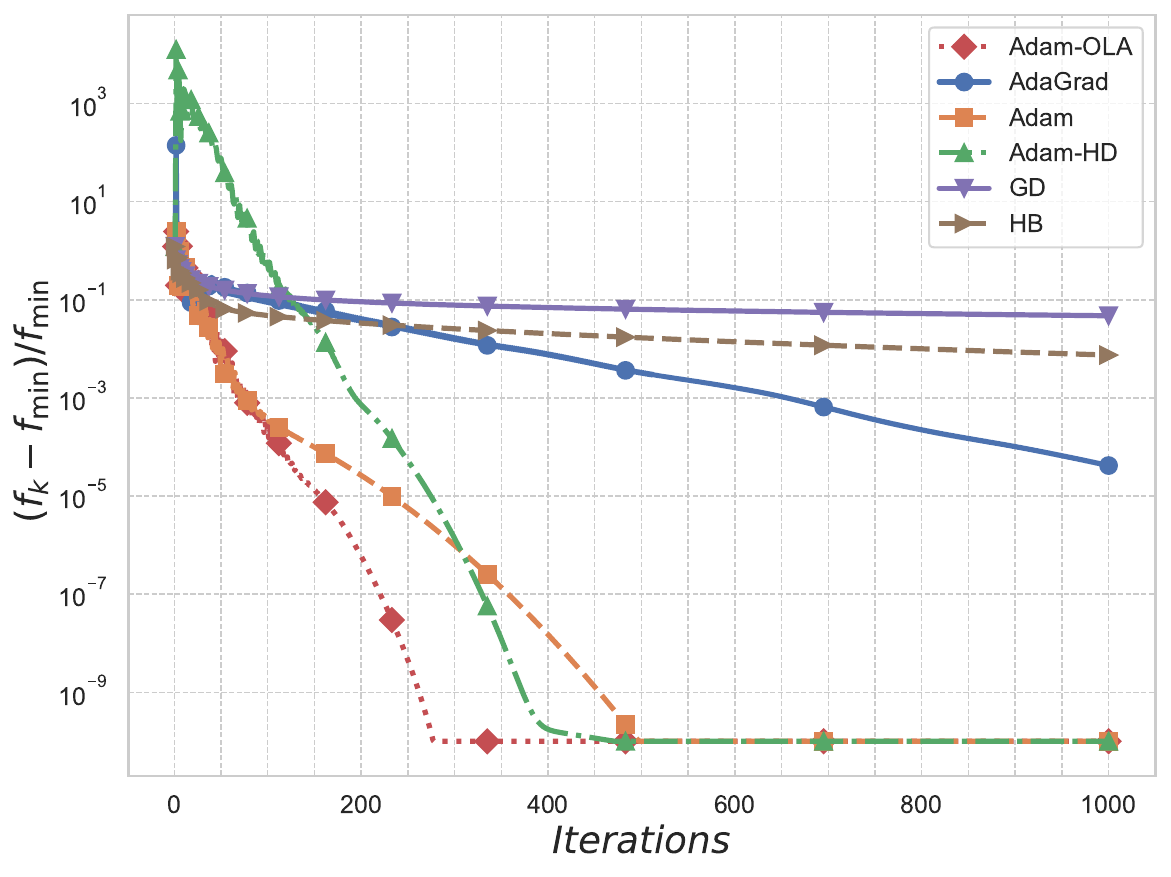}
        \caption{\texttt{a1a}}
        \label{fig:svm_a1a}
    \end{subfigure}
    \hfill
    \begin{subfigure}[b]{0.49\textwidth}
        \centering
        \includegraphics[width=\linewidth]{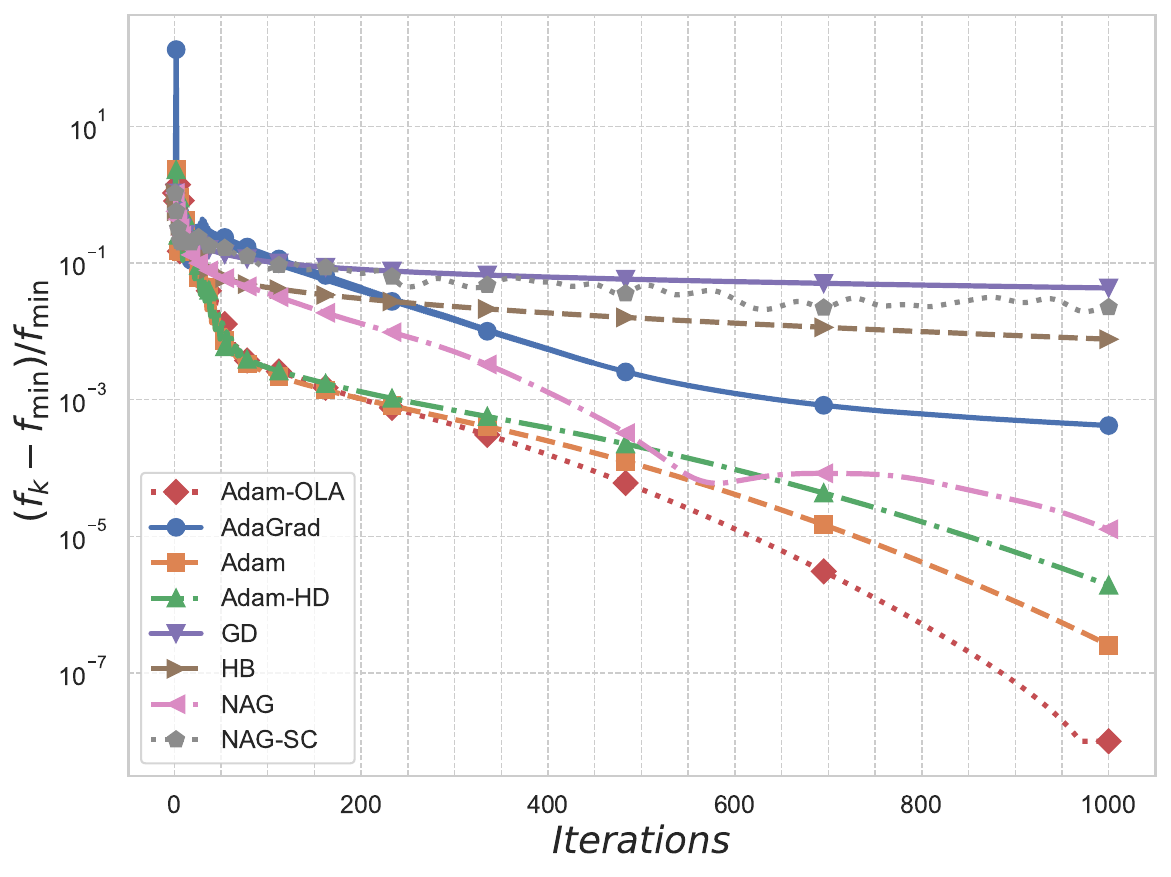}
        \caption{\texttt{a2a}}
        \label{fig:svm_a2a}
    \end{subfigure}
    % \hfill
    % \begin{subfigure}[b]{0.32\textwidth}
    %     \centering
    %     \includegraphics[width=\linewidth]{fig/exp_svm_a3a.pdf}
    %     % \caption{Logistic: Train on synthetic problems and test on synthetic ones.}
    %     \label{fig:svm_a3a}
    % \end{subfigure}
    \caption{Comparison of different optimizers on smooth SVM: Function value versus iteration. Here, $f_{\min}$ denotes the minimum function value achieved across all iterations for each optimizer.}
    \label{fig:ola}
\end{figure}

\section{Conclusion}
\label{sec:conclusion}

In this work, we introduced the concept of a differentiable discrete stopping time for iterative algorithms, establishing a link between continuous time dynamics and their discrete approximations. We proposed an efficient method using the discrete adjoint principle to compute the sensitivity of the discrete stopping time. Our experiments demonstrate that this approach provides an accurate gradient approximation while requiring substantially fewer function evaluations for the forward pass compared to methods relying on continuous ODE solves, proving efficient and scalable for high-dimensional problems. This allows direct optimization of algorithms for convergence speed, with potential applications in L2O and online adaptation.

However, we note that employing a forward Euler discretization with a fixed time step may be too coarse for the algorithmic design. This limitation is also reflected in the error bound estimated in Theorem \ref{thm:approx-error}. In future work, we plan to explore more tailored algorithmic designs for $\Acal$ alongside more sophisticated discretization schemes—such as symplectic integrators or methods that incorporate higher-order information. Such approaches may enable more accurate control of the global error and allow for a wider range of stable time steps during discretization.

\bibliographystyle{unsrtnat}
\bibliography{ref}

\begin{thebibliography}{34}
\providecommand{\natexlab}[1]{#1}
\providecommand{\url}[1]{\texttt{#1}}
\expandafter\ifx\csname urlstyle\endcsname\relax
  \providecommand{\doi}[1]{doi: #1}\else
  \providecommand{\doi}{doi: \begingroup \urlstyle{rm}\Url}\fi

\bibitem[Taha(1997)]{taha1997operations}
Hamdy~A. Taha.
\newblock \emph{Operations research: an introduction}.
\newblock Prentice Hall, 6 edition, 1997.

\bibitem[Zhao et~al.(2025)Zhao, Zhou, Li, Tang, Wang, Hou, Min, Zhang, Zhang,
  Dong, Du, Yang, Chen, Chen, Jiang, Ren, Li, Tang, Liu, Liu, Nie, and
  Wen]{zhao2025surveylargelanguagemodels}
Wayne~Xin Zhao, Kun Zhou, Junyi Li, Tianyi Tang, Xiaolei Wang, Yupeng Hou,
  Yingqian Min, Beichen Zhang, Junjie Zhang, Zican Dong, Yifan Du, Chen Yang,
  Yushuo Chen, Zhipeng Chen, Jinhao Jiang, Ruiyang Ren, Yifan Li, Xinyu Tang,
  Zikang Liu, Peiyu Liu, Jian-Yun Nie, and Ji-Rong Wen.
\newblock A survey of large language models, 2025.
\newblock URL \url{https://arxiv.org/abs/2303.18223}.

\bibitem[Pilbeam(2018)]{pilbeam2018finance}
Keith Pilbeam.
\newblock \emph{Finance and financial markets}.
\newblock Bloomsbury Publishing, 2018.

\bibitem[Feurer and Hutter(2019)]{feurer2019hyperparameter}
Matthias Feurer and Frank Hutter.
\newblock Hyperparameter optimization.
\newblock In Frank Hutter, Lars Kotthoff, and Joaquin Vanschoren, editors,
  \emph{Automated Machine Learning: Methods, Systems, Challenges}, pages 3--33.
  Springer International Publishing, 2019.
\newblock \doi{10.1007/978-3-030-05318-5_1}.

\bibitem[Chen et~al.(2024)Chen, Liu, and Yin]{chen2024learning}
Xiaohan Chen, Jialin Liu, and Wotao Yin.
\newblock Learning to optimize: A tutorial for continuous and mixed-integer
  optimization.
\newblock \emph{Science China Mathematics}, 67\penalty0 (6):\penalty0
  1191--1262, 2024.

\bibitem[Nemirovskij and Yudin(1983)]{nemirovskij1983problem}
Arkadij~Semenovi{\v{c}} Nemirovskij and David~Borisovich Yudin.
\newblock \emph{Problem complexity and method efficiency in optimization}.
\newblock Wiley-Interscience, 1983.

\bibitem[Su et~al.(2016)Su, Boyd, and Candes]{su2016differential}
Weijie Su, Stephen Boyd, and Emmanuel~J Candes.
\newblock A differential equation for modeling nesterov's accelerated gradient
  method: Theory and insights.
\newblock \emph{Journal of Machine Learning Research}, 17\penalty0
  (1):\penalty0 5312--5354, 2016.

\bibitem[Shi et~al.(2022)Shi, Du, Jordan, and Su]{shi2022understanding}
Bin Shi, Simon~S Du, Michael~I Jordan, and Weijie~J Su.
\newblock Understanding the acceleration phenomenon via high-resolution
  differential equations.
\newblock \emph{Mathematical Programming}, 195:\penalty0 79--148, 2022.
\newblock \doi{10.1007/s10107-021-01681-8}.

\bibitem[Shi et~al.(2019)Shi, Du, Su, and Jordan]{shi2019acceleration}
Bin Shi, Simon~S Du, Weijie Su, and Michael~I Jordan.
\newblock Acceleration via symplectic discretization of high-resolution
  differential equations.
\newblock In \emph{Advances in Neural Information Processing Systems}, pages
  5745--5753, 2019.

\bibitem[Laborde and Oberman(2020)]{laborde2020lyapunov}
Maxime Laborde and Adam Oberman.
\newblock A lyapunov analysis for accelerated gradient methods: From
  deterministic to stochastic case.
\newblock In \emph{International Conference on Artificial Intelligence and
  Statistics}, volume 108, pages 602--612. PMLR, 2020.

\bibitem[Li et~al.(2025)Li, Zhu, Xie, and Wen]{li2025angd}
Chenyi Li, Shuchen Zhu, Zhonglin Xie, and Zaiwen Wen.
\newblock Accelerated natural gradient method for parametric manifold
  optimization, 2025.
\newblock URL \url{https://arxiv.org/abs/2504.05753}.

\bibitem[Blondel et~al.(2022)Blondel, Berthet, Cuturi, Frostig, Hoyer,
  Llinares-L{\'o}pez, Pedregosa, and Vert]{blondel2022efficient}
Mathieu Blondel, Quentin Berthet, Marco Cuturi, Roy Frostig, Stephan Hoyer,
  Felipe Llinares-L{\'o}pez, Fabian Pedregosa, and Jean-Philippe Vert.
\newblock Efficient and modular implicit differentiation.
\newblock In \emph{Advances in Neural Information Processing Systems},
  volume~35, pages 5230--5242, 2022.

\bibitem[Bolte et~al.(2021)Bolte, Le, Pauwels, and
  Silveti-Falls]{bolte2021nonsmooth}
J{\'e}r{\^o}me Bolte, Tam Le, Edouard Pauwels, and Tony Silveti-Falls.
\newblock Nonsmooth implicit differentiation for machine-learning and
  optimization.
\newblock In \emph{Advances in Neural Information Processing Systems},
  volume~34, pages 13537--13549, 2021.

\bibitem[Chang et~al.(2022)Chang, Griffiths, and Levine]{chang2022object}
Michael Chang, Thomas~L. Griffiths, and Sergey Levine.
\newblock Object representations as fixed points: Training iterative refinement
  algorithms with implicit differentiation.
\newblock In \emph{Advances in Neural Information Processing Systems},
  volume~35, pages 32694--32708, 2022.

\bibitem[Bertrand et~al.(2022)Bertrand, Klopfenstein, Massias, Blondel, Vaiter,
  Gramfort, and Salmon]{bertrand2022implicit}
Quentin Bertrand, Quentin Klopfenstein, Mathurin Massias, Mathieu Blondel,
  Samuel Vaiter, Alexandre Gramfort, and Joseph Salmon.
\newblock Implicit differentiation for fast hyperparameter selection in
  non-smooth convex learning.
\newblock \emph{Journal of Machine Learning Research}, 23\penalty0
  (149):\penalty0 1--43, 2022.

\bibitem[Geng et~al.(2021)Geng, Zhang, Bai, Wang, and Lin]{geng2021training}
Zhengyang Geng, Xin-Yu Zhang, Shaojie Bai, Yisen Wang, and Zhouchen Lin.
\newblock On training implicit models.
\newblock In \emph{Advances in Neural Information Processing Systems},
  volume~34, pages 24247--24260, 2021.

\bibitem[Vicol et~al.(2022)Vicol, Lorraine, Pedregosa, Duvenaud, and
  Grosse]{vicol2022implicit}
Paul Vicol, Jonathan~P. Lorraine, Fabian Pedregosa, David Duvenaud, and
  Roger~B. Grosse.
\newblock On implicit bias in overparameterized bilevel optimization.
\newblock In \emph{International Conference on Machine Learning}, volume 162,
  pages 22234--22259. PMLR, 2022.

\bibitem[Xu et~al.(2023)Xu, Molloy, and Gould]{xu2023revisiting}
Ming Xu, Timothy~L Molloy, and Stephen Gould.
\newblock Revisiting implicit differentiation for learning problems in optimal
  control.
\newblock In \emph{Advances in Neural Information Processing Systems},
  volume~36, pages 60060--60076, 2023.

\bibitem[Chen et~al.(2022{\natexlab{a}})Chen, Chen, Chen, Heaton, Liu, Wang,
  and Yin]{chen2022learning}
Tianlong Chen, Xiaohan Chen, Wuyang Chen, Howard Heaton, Jialin Liu, Zhangyang
  Wang, and Wotao Yin.
\newblock Learning to optimize: A primer and a benchmark.
\newblock \emph{Journal of Machine Learning Research}, 23\penalty0
  (189):\penalty0 1--59, 2022{\natexlab{a}}.

\bibitem[Chen et~al.(2022{\natexlab{b}})Chen, Chen, Cheng, Chen, Awadallah, and
  Wang]{chen2022scalable}
Xuxi Chen, Tianlong Chen, Yu~Cheng, Weizhu Chen, Ahmed Awadallah, and Zhangyang
  Wang.
\newblock Scalable learning to optimize: A learned optimizer can train big
  models.
\newblock In \emph{European Conference on Computer Vision}, pages 389--405.
  Springer, 2022{\natexlab{b}}.
\newblock \doi{10.1007/978-3-031-20050-2_23}.

\bibitem[Chen et~al.(2020)Chen, Zhang, Zhou, Chang, Liu, Amini, and
  Wang]{chen2020training}
Tianlong Chen, Weiyi Zhang, Jingyang Zhou, Shiyu Chang, Sijia Liu, Lisa Amini,
  and Zhangyang Wang.
\newblock Training stronger baselines for learning to optimize.
\newblock In \emph{Advances in Neural Information Processing Systems},
  volume~33, pages 7332--7343, 2020.

\bibitem[Yang et~al.(2023{\natexlab{a}})Yang, Chen, Zhu, He, Tao, Liang, and
  Wang]{yang2023learning}
Junjie Yang, Tianlong Chen, Mingkang Zhu, Fengxiang He, Dacheng Tao, Yingbin
  Liang, and Zhangyang Wang.
\newblock Learning to generalize provably in learning to optimize.
\newblock In \emph{Proceedings of the 26th International Conference on
  Artificial Intelligence and Statistics}, volume 206, pages 9807--9825. PMLR,
  2023{\natexlab{a}}.

\bibitem[Yang et~al.(2023{\natexlab{b}})Yang, Chen, Chen, Wang, and
  Liang]{yang2023ml2o}
Junjie Yang, Xuxi Chen, Tianlong Chen, Zhangyang Wang, and Yingbin Liang.
\newblock {M-L2O}: Towards generalizable learning-to-optimize by test-time fast
  self-adaptation.
\newblock In \emph{International Conference on Learning Representations},
  2023{\natexlab{b}}.
\newblock URL \url{https://openreview.net/forum?id=s7oOe6cNRT8}.

\bibitem[Song et~al.(2024)Song, Lin, Wang, and Xu]{song2024towards}
Qingyu Song, Wei Lin, Juncheng Wang, and Hong Xu.
\newblock Towards robust learning to optimize with theoretical guarantees.
\newblock In \emph{Proceedings of the IEEE/CVF Conference on Computer Vision
  and Pattern Recognition}, pages 27498--27506, 2024.

\bibitem[Xie et~al.(2024)Xie, Yin, and Wen]{xie2024ode}
Zhonglin Xie, Wotao Yin, and Zaiwen Wen.
\newblock {ODE-based Learning to Optimize}, 2024.
\newblock URL \url{https://arxiv.org/abs/2406.02006}.

\bibitem[Bolte and Pauwels(2021)]{bolte2021conservative}
J{\'e}r{\^o}me Bolte and Edouard Pauwels.
\newblock Conservative set valued fields, automatic differentiation, stochastic
  gradient methods and deep learning.
\newblock \emph{Mathematical Programming}, 188:\penalty0 19--51, 2021.

\bibitem[Chen et~al.(2018)Chen, Rubanova, Bettencourt, and
  Duvenaud]{chen2018neural}
Ricky T.~Q. Chen, Yulia Rubanova, Jesse Bettencourt, and David Duvenaud.
\newblock Neural ordinary differential equations.
\newblock In \emph{Advances in Neural Information Processing Systems},
  volume~31, 2018.

\bibitem[Andrychowicz et~al.(2016)Andrychowicz, Denil, G{\'o}mez, Hoffman,
  Pfau, Schaul, Shillingford, and de~Freitas]{l2o-dm}
Marcin Andrychowicz, Misha Denil, Sergio G{\'o}mez, Matthew~W. Hoffman, David
  Pfau, Tom Schaul, Brendan Shillingford, and Nando de~Freitas.
\newblock Learning to learn by gradient descent by gradient descent.
\newblock In Daniel~D. Lee, Masashi Sugiyama, Ulrike von Luxburg, Isabelle
  Guyon, and Roman Garnett, editors, \emph{Advances in Neural Information
  Processing Systems 29: Annual Conference on Neural Information Processing
  Systems 2016, December 5-10, 2016, Barcelona, Spain}, pages 3981--3989, 2016.
\newblock URL
  \url{https://proceedings.neurips.cc/paper/2016/hash/fb87582825f9d28a8d42c5e5e5e8b23d-Abstract.html}.

\bibitem[Lv et~al.(2017)Lv, Jiang, and Li]{rnnprop}
Kaifeng Lv, Shunhua Jiang, and Jian Li.
\newblock Learning gradient descent: Better generalization and longer horizons.
\newblock In Doina Precup and Yee~Whye Teh, editors, \emph{Proceedings of the
  34th International Conference on Machine Learning, {ICML} 2017, Sydney, NSW,
  Australia, 6-11 August 2017}, volume~70 of \emph{Proceedings of Machine
  Learning Research}, pages 2247--2255. {PMLR}, 2017.
\newblock URL \url{http://proceedings.mlr.press/v70/lv17a.html}.

\bibitem[Lee and Mangasarian(2001)]{ssvm}
Yuh{-}Jye Lee and O.~L. Mangasarian.
\newblock {SSVM:} {A} smooth support vector machine for classification.
\newblock \emph{Comput. Optim. Appl.}, 20\penalty0 (1):\penalty0 5--22, 2001.
\newblock \doi{10.1023/A:1011215321374}.
\newblock URL \url{https://doi.org/10.1023/A:1011215321374}.

\bibitem[Chang and Lin(2011)]{libsvm}
Chih-Chung Chang and Chih-Jen Lin.
\newblock {LIBSVM}: A library for support vector machines.
\newblock \emph{ACM Transactions on Intelligent Systems and Technology},
  2:\penalty0 27:1--27:27, 2011.
\newblock Software available at \url{http://www.csie.ntu.edu.tw/~cjlin/libsvm}.

\bibitem[Baydin et~al.(2018)Baydin, Cornish, Mart{\'{\i}}nez{-}Rubio, Schmidt,
  and Wood]{atilim2018online}
Atilim~Gunes Baydin, Robert Cornish, David Mart{\'{\i}}nez{-}Rubio, Mark
  Schmidt, and Frank Wood.
\newblock Online learning rate adaptation with hypergradient descent.
\newblock In \emph{6th International Conference on Learning Representations,
  {ICLR} 2018, Vancouver, BC, Canada, April 30 - May 3, 2018, Conference Track
  Proceedings}. OpenReview.net, 2018.
\newblock URL \url{https://openreview.net/forum?id=BkrsAzWAb}.

\bibitem[Chu et~al.(2025)Chu, Gao, Ye, and Udell]{chu2025provable}
Ya-Chi Chu, Wenzhi Gao, Yinyu Ye, and Madeleine Udell.
\newblock Provable and practical online learning rate adaptation with
  hypergradient descent, 2025.
\newblock URL \url{https://arxiv.org/abs/2502.11229}.

\bibitem[Liu et~al.(2023)Liu, Chen, Wang, Yin, and Cai]{liu2023towards}
Jialin Liu, Xiaohan Chen, Zhangyang Wang, Wotao Yin, and HanQin Cai.
\newblock Towards constituting mathematical structures for learning to
  optimize.
\newblock In \emph{Proceedings of the 40th International Conference on Machine
  Learning}, pages 21426--21449, 2023.

\end{thebibliography}

\newpage
\appendix
\section{Proof of Theorem \ref{thm:differentiability}}
\label{app:differentiability}
\begin{proof}
    Consider a function $ G(\theta, x_0, t) = J(x(t; \theta, x_0)) - \epsilon $. By the definition of $ T = T_J(\theta, x_0, \epsilon) $, we have
    \[
    G(\theta, x_0, T) = G(\theta, x_0, T_J(\theta, x_0, \epsilon)) = 0.
    \]
    Computing the partial derivatives yields
    \[
    \frac{\partial G}{\partial \theta} = \nabla J(x)^\top \frac{\partial x}{\partial \theta},
    \]
    and
    \[
    \frac{\partial G}{\partial t} = \nabla J(x)^\top \dot{x}(t).
    \]
    Using the implicit function theorem, we conclude that $ T $ can be expressed locally as a continuously differentiable function of $ \theta $ or $ x_0 $. We now differentiate $ G $ with respect to $ \theta $ and $ x_0 $, which yields
    \[
    0 = \frac{\dd}{\dd \theta} \left( G(\theta, x_0, T) \right)
    = \frac{\dd}{\dd \theta} J\left( x(T(\theta, x_0, \epsilon); \theta, x_0) \right)
    = \nabla J(x)^\top \left( \frac{\partial x}{\partial \theta} + \dot{x}(T) \frac{\partial T}{\partial \theta} \right),
    \]
    and
    \[
    0 = \frac{\dd}{\dd x_0} \left( G(\theta, x_0, T) \right)
    = \frac{\dd}{\dd x_0} J\left( x(T(\theta, x_0, \epsilon); \theta, x_0) \right)
    = \nabla J(x)^\top \left( \frac{\partial x}{\partial x_0} + \dot{x}(T) \frac{\partial T}{\partial x_0} \right).
    \]
    Rearranging these equations leads to
    \[
    \nabla J(x)^\top \dot{x}(T) \frac{\partial T}{\partial \theta}
    = -\nabla J(x)^\top \frac{\partial x}{\partial \theta},
    \quad \text{and hence} \quad
    \frac{\partial T}{\partial \theta}
    = \left( \nabla J(x)^\top \dot{x}(T) \right)^{-1} \nabla J(x)^\top \frac{\partial x}{\partial \theta},
    \]
    as well as
    \[
    \nabla J(x)^\top \dot{x}(T) \frac{\partial T}{\partial x_0}
    = -\nabla J(x)^\top \frac{\partial x}{\partial x_0},
    \quad \text{and hence} \quad
    \frac{\partial T}{\partial x_0}
    = \left( \nabla J(x)^\top \dot{x}(T) \right)^{-1} \nabla J(x)^\top \frac{\partial x}{\partial x_0}.
    \]
    The above equations complete the proof. 
\end{proof}

\section{Proof of Theorem \ref{thm:approx-error}}
\label{app:approx-proof}

We first present a basic analysis in numerical ODEs. 

\begin{proposition}[Error analysis of  the forward Euler method]
\label{prop:forward_Euler_error}
Let $f:\mathbb{R}^n\times\mathbb{R}\rightarrow \mathbb{R}$ be a function defined by $(x,t)\mapsto f(x,t)$. Suppose the following assumptions hold
\begin{enumerate}
    \item There exists a constant $ L_x > 0 $ such that $ \| f(x_1, t) - f(x_2, t) \| \leq L_x \| x_1 - x_2 \| $ for all $ x_1, x_2 $, and $ t $.
    \item There exists a constant $ L_t > 0 $ such that $ \| f(x, t_1) - f(x, t_2) \| \leq L_t | t_1 - t_2 | $ for all $ x, t_1 $, and $ t_2 $.
    \item There exists a constant $ M > 0 $ such that $ \| f(x, t) \| < M $ for all $ x $ and $ t $.
\end{enumerate}

Given an initial condition $ x(t_0) = x_0 $ and a fixed stepsize $ h $, we consider the sequence generated by the forward Euler method as
\[
x_{k+1} = x_k + h f(x_k, t_k), \quad t_k = t_0 + k h.
\]

Then, for any positive integer $k$, the error $e_k = x_k - x(t_k)$ satisfies
\[
\|e_k\| \leq \frac{h}{2} \left( M + \frac{L_t}{L_x} \right) \left( e^{L_x h k} - 1 \right).
\]
\end{proposition}

\begin{proof}
    We begin by expressing the error at step $ k+1 $ as
    \[
    e_{k+1} = x_{k+1} - x(t_{k+1}) = e_k + h \left( f(x_k, t_k) - f(x(t_k), t_k) \right) + x(t_k) + h f(x(t_k), t_k) - x(t_{k+1}).
    \]

    Applying Lipschitz continuity, we obtain the inequality
    \[
    \|e_{k+1}\| \leq (1 + L_x h) \|e_k\| + \|x(t_k) + h f(x(t_k), t_k) - x(t_{k+1})\|.
    \]

    The second term on the right-hand side can be expressed in integral form as
    \[
    \|x(t_k) + h f(x(t_k), t_k) - x(t_{k+1})\| = \left\| \int_{t_k}^{t_{k+1}} \left[ f(x(t), t) - f(x(t_k), t_k) \right] dt \right\|,
    \]
    which is bounded above by
    \[
    \left\| \int_{t_k}^{t_{k+1}} \left[ f(x(t), t_k) - f(x(t_k), t_k) \right] dt \right\| + \left\| \int_{t_k}^{t_{k+1}} \left[ f(x(t), t) - f(x(t), t_k) \right] dt \right\|.
    \]

    Substituting the assumptions, we estimate the first integral as 
    \[
    \left\| \int_{t_k}^{t_{k+1}} \left[ f(x(t), t_k) - f(x(t_k), t_k) \right] dt \right\| \leq L_x \int_{t_k}^{t_{k+1}} \|x(t) - x(t_k)\| dt \leq \frac{1}{2} M L_x h^2,
    \]
    where the last inequality follows from the Lagrange mean value theorem, which implies that 
    \begin{align*}
    \int_{t_k}^{t_{k+1}} \|x(t) - x(t_k)\| dt 
    &= \int_{t_k}^{t_{k+1}} \|\dot{x}(\xi)\| |t - t_k| dt \\
    &= \int_{t_k}^{t_{k+1}} \|f(x(\xi), \xi)\| |t - t_k| dt \\
    &\leq M \int_{t_k}^{t_{k+1}} |t - t_k| dt = \frac{1}{2} M h^2.
    \end{align*}

    Similarly, for the second integral, we derive the bound as
    \[
    \left\| \int_{t_k}^{t_{k+1}} \left[ f(x(t), t) - f(x(t), t_k) \right] dt \right\| \leq L_t \int_{t_k}^{t_{k+1}} |t - t_k| dt = \frac{1}{2} L_t h^2.
    \]

    Combining these inequalities, we obtain
    \[
    \|e_{k+1}\| \leq (1 + L_x h) \|e_k\| + \frac{h^2}{2} (L_t + M L_x).
    \]

    Finally, using the initial error $ e_0 = 0 $, we conclude that the global error satisfies 
    \[
    \|e_k\| \leq \frac{h}{2} \left( M + \frac{L_t}{L_x} \right) \left( e^{L_x h k} - 1 \right).
    \]
    This completes the proof.
\end{proof}

\begin{proof}[Proof of the Theorem]
    The following conditions are assumed throughout our analysis. First, the function $\mathcal{A}$ is twice continuously differentiable, i.e., $\mathcal{A} \in C^2$. Second, $\mathcal{A}$ itself, together with all partial derivatives of $\mathcal{A}$ (such as $\frac{\partial}{\partial x}\mathcal{A}$, $\frac{\partial^2\mathcal{A}}{\partial \theta\partial t}$) and the gradient and Hessian of $J$ (i.e., $\nabla J$ and $\nabla^2 J$), are uniformly bounded by constants $A, A_x, A_{\theta}, A_t, A_{\theta,x}, A_{x,x}, A_{x,t}, A_{\theta,t}, J_1,$ and $J_2$, respectively. Third, we assume the boundary condition $|\nabla J(x(T))^\top\dot{x}(T)| = \delta$.
    
    For clarity and brevity, our main theorem states the regularity and boundedness assumptions using Sobolev norms; specifically, we require that $\mathcal{A}(\theta,x(t),t)$, regarded as a function of $(\theta, t)$, has uniformly bounded $W^{2,\infty}$ norms with respect to $(\theta, t)$ in a neighborhood of $\theta \times [t_0,\; t_0 + N_J h]$, and that $J$, regarded as a function of $x$, has a uniformly bounded $W^{2,\infty}$ norm in a neighborhood of $x(T_J)$. In this proof, we equivalently expand these assumptions by explicitly introducing uniform bounds for $\mathcal{A}$, its partial derivatives (with respect to $x$, $\theta$, $t$, etc.), and for the gradient $\nabla J$ and Hessian $\nabla^2 J$, denoted by $A, A_x, A_{\theta}$, $A_t$, $A_{\theta,x}$, $A_{x,x}$, $A_{x,t}$, $A_{\theta,t}$, $J_1$, and $J_2$, respectively. This explicit formulation is purely for notational convenience in the analysis, as it allows us to refer directly to these quantities in the derivations, especially during Taylor expansions and error estimates. We emphasize that these detailed bounds can be derived from the $W^{2,\infty}$ norm boundedness assumed in the theorem statement.
    
    Without loss of generality, we only prove the case for the $L^2$ norm. We first recall the form and the definition of the derivative. They are given by
    \[
    \nabla_{\theta} T = \nabla_{\theta}T_J(\theta,x_0,\epsilon)=
    -\frac{\nabla J(x(T))^\top\frac{\partial x(T)}{\partial \theta}}{\nabla J(x(T))^\top\dot{x}(T)},
    \]\[
    \nabla_{\theta}N=
    \nabla_{\theta}N_J(\theta,x_0,\epsilon)=
    -\frac{h\nabla J(x_N)^\top\frac{\partial x_N}{\partial \theta}}{J(x_N)-J(x_{N-1})}.
    \]

    We consider the iteration 
    \[
    x_{k+1} = x_k - h\mathcal{A}(\theta,x_k,t_k).
    \]

    Differentiating with respect to $\theta$, we obtain 
    \[
        \frac{\partial x_{k+1}}{\partial \theta} = 
        \left( I - h\frac{\partial}{\partial x}\mathcal{A}(\theta,x_k,t_k) \right)\frac{\partial x_k}{\partial \theta} - h\frac{\partial}{\partial \theta}\mathcal{A}(\theta,x_k,t_k),
    \]
    where the initial condition $\frac{\partial x_0}{\partial \theta}$ holds.

    Also, we consider the flow 
    \[
    \dot{x}(t) = -\mathcal{A}(\theta,x(t),t).
    \]
    Differentiating with respect to $\theta$, we obtain
    \begin{equation}\label{eq:ODE_of_partial_theta}
        \frac{\dd }{\dd t}
        \frac{\partial x(t)}{\partial \theta} = -\frac{\partial}{\partial \theta}\mathcal{A}(\theta,x(t),t) - 
        \frac{\partial}{\partial x}\mathcal{A}(\theta,x(t),t)\frac{\partial x(t)}{\partial \theta}, 
    \end{equation}
    where the initial condition $\frac{\partial x(t_0)}{\partial \theta} = 0$ holds. 
    Let $u(t) = \frac{\partial x(t)}{\partial \theta}$. It is easy to observe that the iteration above corresponds to the forward Euler method for solving the ODE
    \[
    \frac{\dd}{\dd t}u(t) = 
    -\frac{\partial}{\partial \theta}\mathcal{A}(\theta,x(t),t) - 
        \frac{\partial}{\partial x}\mathcal{A}(\theta,x(t),t)u(t).  
    \]

    We now proceed to show that $u(t)$, for $t \in [t_0,T]$, is bounded by some constant $M > 0$. Let $v = u^\top u$, $B(t) = -\frac{\partial}{\partial \theta}\mathcal{A}(\theta,x(t),t)$, and $C(t) = -\frac{\partial}{\partial x}\mathcal{A}(\theta,x(t),t)$. Then we can derive that
    \[
        \frac{\dd}{\dd t}v = 2u^\top \frac{\dd}{\dd t}u = 2u^\top B + 2u^\top C u.
    \]
    Therefore, we have the bound
    \[
        \left|\frac{\dd}{\dd t}v\right| \leq 2\|B\|\sqrt{v} + 2\|C\|v \leq \|B\| + (\|B\| + 2\|C\|)v \leq A_{\theta} + (A_{\theta} + 2A_x)v.
    \]
    Applying the Gronwall inequality,  for every $t \in [t_0,T]$, we obtain the following estimate 
    \begin{equation}\label{eq:err_u}
        \|u(t)\| = \sqrt{v(t)} \leq \sqrt{\frac{A_{\theta}}{A_{\theta} + 2A_x} \left(e^{(A_{\theta} + 2A_x)(T - t_0)} - 1\right)} \triangleq M.
    \end{equation}

    Employing Proposition~\ref{prop:forward_Euler_error}, we obtain that $\left\|\frac{\partial x_N}{\partial \theta} - \frac{\partial x(Nh)}{\partial \theta}\right\|$ is bounded by 
    \[
        \frac{h}{2}
        \left(
        MA_x+A_{\theta}+
        \frac{M(A_{x,t}+AA_{x,x})+A_{\theta,t}+AA_{\theta,x}}{A_x}
        \right)
        \left(e^{A_x(T+1-t_0)} - 1\right).
    \]

    Let \begin{equation}\label{eq:def_c1}
        c_1 \triangleq\frac{1}{2}
        \left(
        MA_x+A_{\theta}+
        \frac{M(A_{x,t}+AA_{x,x})+A_{\theta,t}+AA_{\theta,x}}{A_x}
        \right)
        \left(e^{A_x(T+1-t_0)} - 1\right).
    \end{equation}

    Noticing that $\frac{\dd}{\dd t}\frac{\partial x(t)}{\partial \theta}$ is bounded by $A_{\theta}+A_xM$ according to \eqref{eq:ODE_of_partial_theta}, and that $|T - Nh| \leq h$, we deduce that
    \[
        \left\|\frac{\partial x(Nh)}{\partial \theta} - \frac{\partial x(T)}{\partial \theta}\right\| \leq (A_{\theta}+A_xM) h.
    \]
    Let $e_1 = \frac{\partial x_N}{\partial \theta} - \frac{\partial x(T)}{\partial \theta}$. Then we obtain the estimate
    \begin{equation}\label{eq:err_of_e1}
        \|e_1\| \leq (A_{\theta}+A_xM + c_1)h.
    \end{equation}

    Similarly, by Proposition~\ref{prop:forward_Euler_error}, we know that
    \[
        \|x_N - x(Nh)\| \leq \frac{h}{2} \left(A + \frac{A_t}{A_x} \right) \left(e^{A_x(T+1-t_0)} - 1\right).
    \]

    Let
    \begin{equation}\label{eq:def_c2}
        c_2\triangleq \frac{1}{2} \left(A + \frac{A_t}{A_x} \right) \left(e^{A_x(T+1-t_0)} - 1\right).
    \end{equation}

    Since $\frac{\dd}{\dd t}x(t) = -\mathcal{A}(\theta,x(t),t)$ is bounded by $A$ and $|T - Nh| \leq h$, it follows that
    \[
        \|x_N - x(T)\| \leq (A + c_2)h.
    \]
    Let $e_2 = \nabla J(x_N) - \nabla J(x(T))$. Since $\|\nabla^2 J\|$ is bounded by $J_2$, we obtain the estimate
    \begin{equation}\label{eq:err_of_e2}
        |e_2| \leq J_2 (A + c_2) h.
    \end{equation}

    The Taylor expansion yields
    \[
        J(x_N) = J(x_{N-1}) - \nabla J(x_{N-1})^\top(x_N - x_{N-1}) + \frac{1}{2}(x_N - x_{N-1})^\top \nabla^2 J(\xi)(x_N - x_{N-1}).
    \]
    Combining this with the fact that $x_N - x_{N-1} = -h \mathcal{A}(\theta,x_{N-1},t_{N-1})$ and that  $\|\nabla^2 J\|$ is bounded by $J_2$, we obtain
    \[
        \left| \frac{J(x_N) - J(x_{N-1})}{h} + \nabla J(x_{N-1})^\top \mathcal{A}(\theta,x_{N-1},t_{N-1}) \right| \leq \frac{1}{2} h A^2 J_2.
    \]

    Let 
    \[
    e_5 = \frac{J(x_N) - J(x_{N-1})}{h} + \nabla J(x_{N-1})^\top \mathcal{A}(\theta,x_{N-1},t_{N-1})
    \]and $e_3 = \nabla J(x_{N-1}) - \nabla J(x(T))$,$e_4 = \mathcal{A}(\theta,x_{N-1},t_{N-1}) + \dot{x}(T)$. We have just derived \begin{equation}
    \label{eq:err_of_e5}
        |e_5|\leq \frac12hA^2J_2
    \end{equation}
    
    As in the previous estimate for $e_2$, we obtain
    \begin{equation}\label{eq:err_of_e3}
        \|e_3\| \leq J_2 (A + c_2) h.
    \end{equation}
    Furthermore, we have
    \begin{align}
        \|e_4\| &\leq \|\mathcal{A}(\theta,x_{N-1},t_{N-1}) - \mathcal{A}(\theta,x(T),t_{N-1})\| + \|\mathcal{A}(\theta,x(T),t_{N-1}) - \mathcal{A}(\theta,x(T),T)\| \notag \\
        &\leq A_x (A + c_2) h + A_t h.\label{eq:err_of_e4}
    \end{align}

    Substituting the definitions of these error terms into the expression for $\nabla_{\theta} N$, we obtain
    \[
        \nabla_{\theta} N = -
        \frac{(\nabla J(x(T)) + e_2)^\top \left( \frac{\partial x(T)}{\partial \theta} + e_1 \right)}{(\nabla J(x(T)) + e_3)^\top (\dot{x}(T) - e_4)+e_5}.
    \]

    Recall that
    \[
        \nabla_{\theta} T = -\frac{\nabla J(x(T))^\top \frac{\partial x(T)}{\partial \theta}}{\nabla J(x(T))^\top \dot{x}(T)}.
    \]

    Comparing these two expressions and combining the estimates from \eqref{eq:err_of_e1}, \eqref{eq:err_of_e2}, \eqref{eq:err_of_e3}, and \eqref{eq:err_of_e4}, together with the assumptions, we arrive at the final estimate that
    \[
        \|\nabla_{\theta} T - \nabla_{\theta} N\| \leq R h + \mathcal{O}(h^2),
    \]
    where
    \begin{align*}
        R = \frac{J_1 M}{\delta^2}& \left( J_1 (A_t + A_x (A + c_2)) + \frac{3}{2} A^2 J_2 + A J_2 c_2 \right)\\
        &\qquad\qquad\qquad\qquad+ \frac{1}{\delta} \left( J_1 (A_0 + A_x M + c_1) + M J_2 (A + c_2) \right).
    \end{align*}
    Here, the constants refer to those defined in \eqref{eq:err_u}, \eqref{eq:def_c1}, \eqref{eq:def_c2}, and the assumptions stated earlier. This completes the proof. 
\end{proof}
\section{Proof of Proposition \ref{thm:adjoint_correctness}}
\label{app:proof-adjoint}

\begin{proof}
We aim to compute $S_{\theta_j} = \nabla J(x_N)^\top \frac{\partial x_N}{\partial \theta_j}$ for each component $\theta_j$ of $\theta$, and $S_{x_0} = \nabla J(x_N)^\top \frac{\partial x_N}{\partial x_0}$.
Let $L(\theta, x_0) = J(x_N(\theta, x_0))$. We are interested in $\nabla_{\theta} L$ and $\nabla_{x_0} L$.
Define the adjoint (co-state) vectors $\lambda_k \in \R^d$ for $k=0, \dots, N$ such that $\lambda_k^\top = \frac{\partial J(x_N)}{\partial x_k} = \nabla J(x_N)^\top \frac{\partial x_N}{\partial x_k}$.
The base case is at $k=N$,
\begin{equation}
    \lambda_N = \frac{\partial J(x_N)}{\partial x_N} = \nabla J(x_N).
    \label{eq:pf_adjoint_init}
\end{equation}
% This corresponds to Line 2 of Algorithm~\ref{alg:discrete_adjoint_computation} where $\lambda$ is initialized.

For $k < N$, $x_N$ depends on $x_k$ through $x_{k+1}$. Using the chain rule
\[ \frac{\partial J(x_N)}{\partial x_k} = \frac{\partial J(x_N)}{\partial x_{k+1}} \frac{\partial x_{k+1}}{\partial x_k}. \]
In terms of our adjoints, we have
\[ \lambda_k^\top = \lambda_{k+1}^\top \frac{\partial x_{k+1}}{\partial x_k}. \]
Given $x_{k+1} = x_k - h\Acal(\theta, x_k, t_k)$, the Jacobian is $\frac{\partial x_{k+1}}{\partial x_k} = I - h \frac{\partial \Acal(\theta, x_k, t_k)}{\partial x_k}$.
Thus, the backward recursion for the adjoints is
\begin{equation}
    \lambda_k^\top = \lambda_{k+1}^\top \left(I - h \frac{\partial \Acal(\theta, x_k, t_k)}{\partial x_k}\right),
    \label{eq:pf_adjoint_recursion}
\end{equation}
or $\lambda_k = \left(I - h \frac{\partial \Acal(\theta, x_k, t_k)}{\partial x_k}\right)^\top \lambda_{k+1}$.
The loop in Algorithm~\ref{alg:discrete_adjoint_computation} implements this recursion. At the beginning of iteration $k$ (loop index in algorithm, representing the step from $x_k$ to $x_{k+1}$), the variable $\lambda$ in the algorithm holds $\lambda_{k+1}$ from our derivation.
% Line 12 updates it to $\lambda_k$.

Now consider the derivative with respect to a parameter $\theta_j$. $J(x_N)$ depends on $\theta_j$ through all $x_m$ for $m \le N$ where $x_m$ is influenced by $\theta_j$. Hence,
\[ \frac{\partial J(x_N)}{\partial \theta_j} = \sum_{m=0}^{N-1} \frac{\partial J(x_N)}{\partial x_{m+1}} \left(\frac{\partial x_{m+1}}{\partial \theta_j}\right)_{\text{explicit}}, \]
where $(\partial x_{m+1}/\partial \theta_j)_{\text{explicit}}$ means differentiating $x_{m+1} = x_m - h\Acal(\theta, x_m, t_m)$ with respect to $\theta_j$ while holding $x_m$ fixed
\[ \left(\frac{\partial x_{m+1}}{\partial \theta_j}\right)_{\text{explicit}} = -h \frac{\partial \Acal(\theta, x_m, t_m)}{\partial \theta_j}. \]
Thus, it holds
\begin{equation}
    \frac{\partial J(x_N)}{\partial \theta_j} = \sum_{m=0}^{N-1} \lambda_{m+1}^\top \left(-h \frac{\partial \Acal(\theta, x_m, t_m)}{\partial \theta_j}\right).
    \label{eq:pf_grad_theta_sum}
\end{equation}
% Line 10 in Algorithm~\ref{alg:discrete_adjoint_computation}, $S_{\theta} \gets S_{\theta} - h (\frac{\partial \Acal_k}{\partial \theta})^\top \lambda$, precisely accumulates this sum. 
The loop runs from $k=N-1$ down to $0$. For each $k$ in the loop, the term added is $-h (\frac{\partial \Acal(\theta, x_k, t_k)}{\partial \theta})^\top \lambda_{k+1}$. Summing these terms gives $\left(\nabla J(x_N)^\top \frac{\partial x_N}{\partial \theta}\right)_j$.
% $\lambda$ corresponds to $\lambda_{k+1}$ (from previous iteration's update or initialization) when Line 10 is executed. 

Finally, for the sensitivity with respect to $x_0$,
\[ \frac{\partial J(x_N)}{\partial x_0} = \lambda_0^\top. \]
After the loop in Algorithm~\ref{alg:discrete_adjoint_computation} finishes (i.e., after the iteration for $k=0$), the variable $\lambda$ will have been updated using $\lambda_1$ and $\frac{\partial \Acal(\theta, x_0, t_0)}{\partial x_0}$, thus holding $\lambda_0$.
% Line 14 correctly assigns this to $S_{x_0}$.
% This completes the proof.
\end{proof}

\section{Details of Experiments}
\label{app:examples}

\textbf{Implementation Details.} 
We adopt the official implementation of \cite{chu2025provable}\footnote{\url{https://github.com/udellgroup/hypergrad}} for the online learning rate adaptation experiments, and the codebase from \cite{liu2023towards}\footnote{\url{https://github.com/xhchrn/MS4L2O}} for L2O experiments. They all follow the MIT License as specified in their respective GitHub repositories. All experiments are conducted on a workstation running Ubuntu with a 12-core Intel Xeon Platinum 8458P CPU (2.7GHz, 44 threads), one NVIDIA RTX 4090 GPU with 24GB memory, and 60GB of RAM. We note that, for both experimental setups, we have made moderate modifications to the original implementations to better align with the goals of our study. However, as the focus of this work is to explore the potential applications of stopping time in optimization rather than to achieve state-of-the-art performance across all settings, we did not perform extensive hyperparameter tuning for the stopping time–based algorithms under different configurations. This choice may explain why our method does not reach SOTA performance in some scenarios.

% \section{Data Synthetic Setting}
% \label{app:l2o-data}
\textbf{NFEs of different solvers.}
Figure \ref{fig:nfe} shows that the NFE for an adaptive solver is mainly influenced by the stopping criterion. Since it does not accept a prespecified time step size, all of the statistics remain the same for different $h$.
\begin{figure}[htbp]
    \centering
    \begin{subfigure}[b]{0.49\textwidth}
        \centering
        \includegraphics[width=\linewidth]{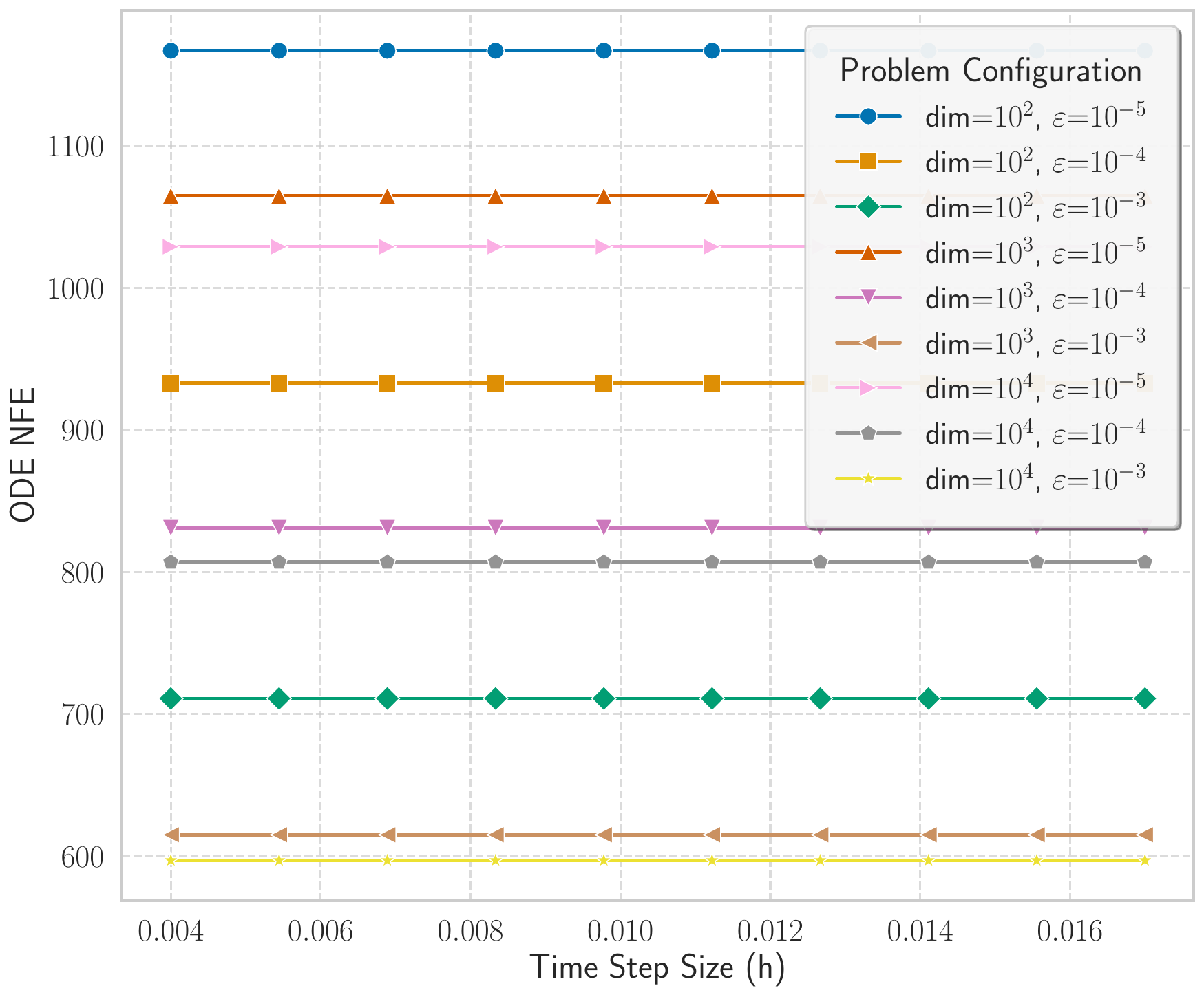}
        \caption{NFE of the adaptive ODE solver.}
        \label{fig:ode-nfe}
    \end{subfigure}
    \hfill
    \begin{subfigure}[b]{0.49\textwidth}
        \centering
        \includegraphics[width=\linewidth]{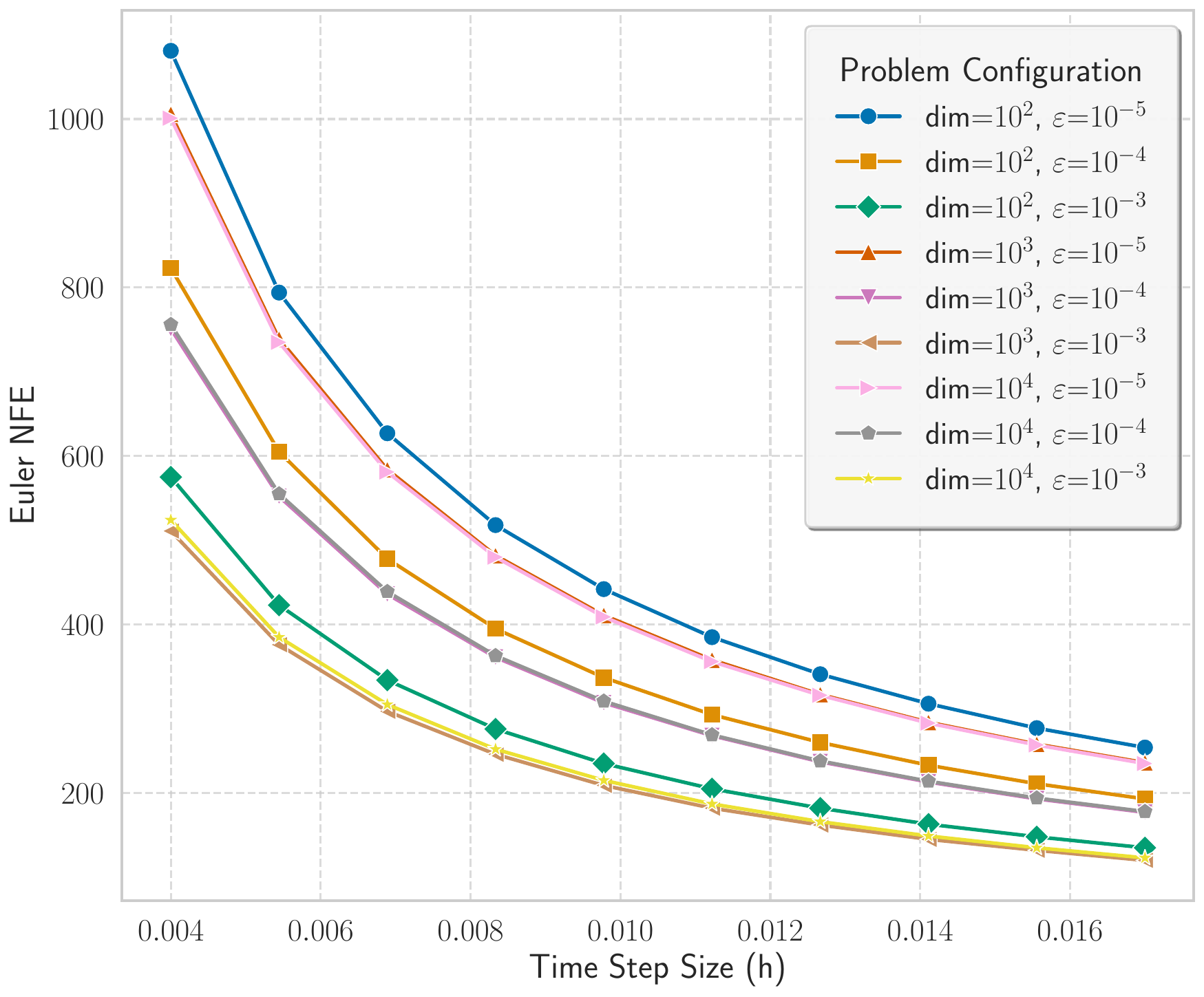}
        \caption{NFE of the Euler discretization.}
        \label{fig:euler-nfe}
    \end{subfigure}
    \caption{NFEs of different solvers.}
    \label{fig:nfe}
\end{figure}

\textbf{Hyperparameters of Baselines.}
\texttt{Adagrad} is an adaptive gradient algorithm that adjusts learning rates per coordinate based on historical gradient information. The learning rate is set $\beta\in \{10^{-3}, 10^{-2}, 10^{-1}, 1.0, 10.0, 1 / L\}$ with $\epsilon = 10^{-8}$. For Heavy-Ball method (\texttt{HB}), the momentum parameter is selected from the set $\{0.1, 0.5, 0.9, 1.0\}$. \texttt{Adam-HD} is a notable variant of Adam~\cite{atilim2018online}, which employs a hypergradient-based scheme to adaptively update the base learning rate at each iteration in an online fashion. For \texttt{Adam-HD}, the hyperparameter $\beta$ used to update the learning rate is chosen from the set $\{10^{-3}, 10^{-4}, 10^{-5}, 10^{-6}\}$. All other abbreviations follow their previously defined roles within the L2O framework. \texttt{Adam-OLA} and \texttt{Adam-HD} are all based on the classical \texttt{Adam}, where $(\beta_1,\beta_2) = (0.9,0.999)$ and $\epsilon = 10^{-8}$. The initial learning rate for Adam is selected from the set $\alpha \in \{10^{-3}, 10^{-2}, 10^{-1}, 1.0, 10.0, 1 / L\}$. $L$ is the Lipschitz constant of $\nabla f(x)$, estimated at the initial point $x_0$. The maximum number of iterations is set to 1000, with a stopping criterion tolerance of $10^{-4}$.
\begin{table}[H]
\centering
\caption{Hyperparameter settings for \texttt{Adam-OLA} on different datasets. The parameter $\beta$ controls the learning rate adaptation magnitude, and $\epsilon$ specifies the sufficient decrease threshold for triggering a learning rate update.}
\label{tab:adam_ola_params}
\begin{tabular}{lcc}
\toprule
\textbf{Dataset (Experiment)} & $\boldsymbol{\beta}$ (Learning Rate Update) & $\boldsymbol{\epsilon}$ (Descent Threshold) \\
\midrule
\texttt{a1a} (\textit{exp\_svm})        & $1 \times 10^{-2}$  & $1 \times 10^{-5}$ \\
\texttt{a2a} (\textit{exp\_svm})        & $1 \times 10^{-3}$ & $1 \times 10^{-3}$ \\
\texttt{a3a} (\textit{exp\_svm})        & $5 \times 10^{-5}$ & $5 \times 10^{-4}$ \\
\texttt{w3a} (\textit{exp\_svm})        & 0.005   & $5 \times 10^{-9}$
\\
\bottomrule
\end{tabular}
\end{table}
% The parameter $\mu_{\text{est}}$ is fixed at 0.001 for all algorithms.
% The Heavy-Ball (\texttt{HB}) method incorporates a momentum coefficient $\beta$, which balances the contribution of the previous update direction with the current gradient. Similarly, \texttt{NAG-SC} employs an acceleration parameter $\mu$ to control its predictive look-ahead mechanism. \texttt{Adagrad} maintains a per-coordinate gradient accumulation matrix $G_{t,i}$, initialized with a small positive constant $\delta$, and includes a smoothing term $\epsilon > 0$ to prevent division by zero. \texttt{Adam-HD} extends standard \texttt{Adam} by introducing a hyper-gradient learning rate $\eta_h$, which is updated as follows:
% $$
%   \eta_{h,t+1}
%   = \eta_{h,t}
%   - \underbrace{\frac{\partial^2 f}{\partial \alpha\,\partial \mathbf{w}}}_{\text{instantaneous hyper-gradient estimate}} \alpha_t,
% $$
% where $\alpha_t$ denotes the base learning rate at iteration $t$. All other L2O baselines employ standard \texttt{Adam} with default decay rates $(\beta_1, \beta_2)=(0.99,0.999)$ and a numerical stability constant $\epsilon$.

\textbf{Formulation of the Smooth SVM.}
In this work, we consider the problem of binary classification using a smooth variant of the SVM, where the non-smooth hinge loss is replaced by its squared counterpart to enable efficient gradient-based optimization. Given a dataset $\{(x_i, y_i)\}_{i=1}^n$ with feature vectors $x_i \in \mathbb{R}^d$ and binary labels $y_i \in \{-1, +1\}$, the objective function takes the form
$$
f(w) = \frac{1}{2} \sum_{i=1}^n \left[\max(0, 1 - y_i w^\top x_i)\right]^2 + \frac{\lambda}{2} \|w\|^2,
$$
where $\lambda > 0$ is a regularization parameter. This formulation preserves the margin-maximizing behavior of the original SVM while allowing for stable and differentiable optimization. We further incorporate an intercept term into the model by appending a constant feature to each input vector. The resulting problem is solved using first-order methods with step size determined via an estimate of the gradient's Lipschitz constant.

\textbf{More Examples of Online Learning Rate Adaptation.}
We report the performance of Algorithm \ref{alg:adam_ola} and other baseline methods. Our method shows consistent improvement in the later stage of the convergence.

% We evaluate Algorithm~\ref{alg:adam_ola} on smooth support vector machine (SVM) problems~\cite{ssvm}, using benchmark datasets from LIBSVM~\cite{libsvm}. The following baselines are considered for comparison: \texttt{HB} denotes the heavy-ball method, and \texttt{NAG-SC} refers to the Nesterov accelerated gradient method tailored for strongly convex objectives. \texttt{Adagrad} is an adaptive gradient algorithm that adjusts learning rates per coordinate based on historical gradient information. 
\begin{figure}[htbp]
    \centering
    \begin{subfigure}[b]{0.49\textwidth}
        \centering
        \includegraphics[width=\linewidth]{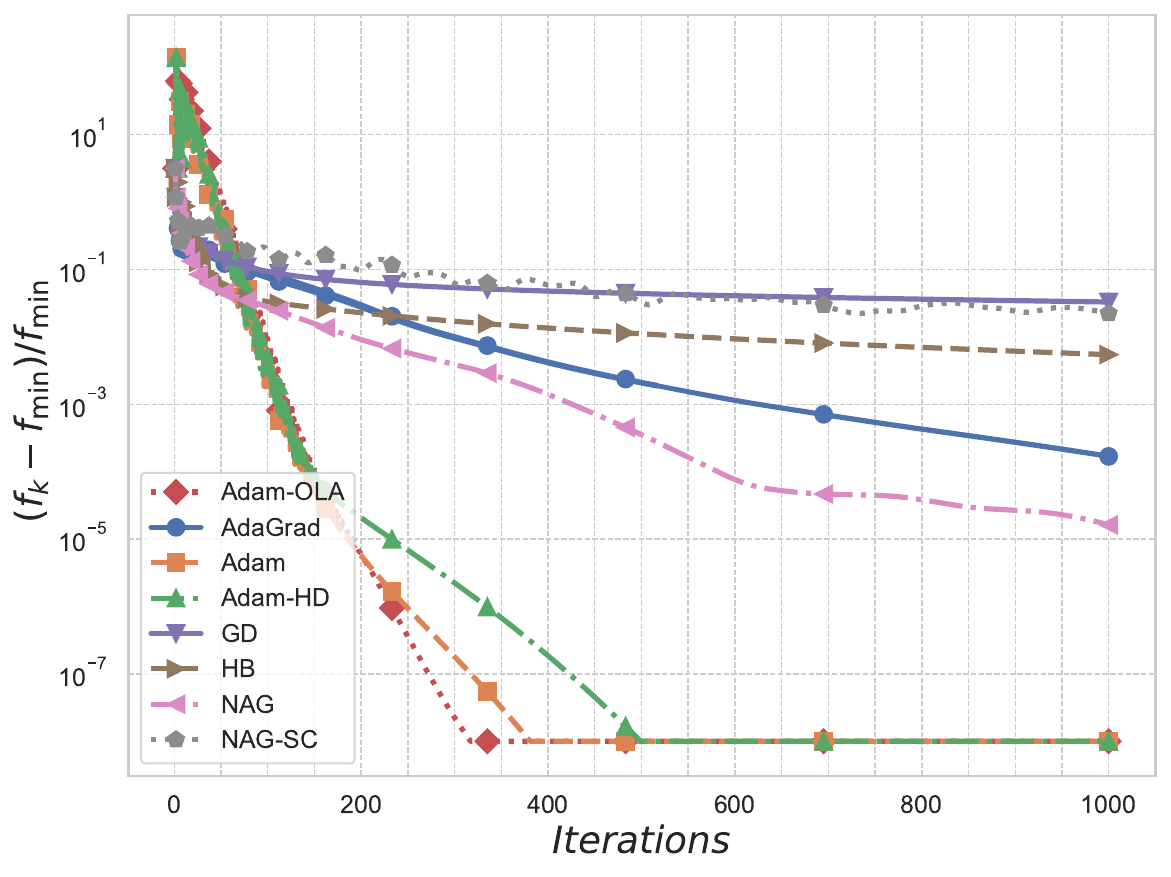}
        \caption{\texttt{a3a}}
        % \label{fig:ode-nfe}
    \end{subfigure}
    \hfill
    \begin{subfigure}[b]{0.49\textwidth}
        \centering
        \includegraphics[width=\linewidth]{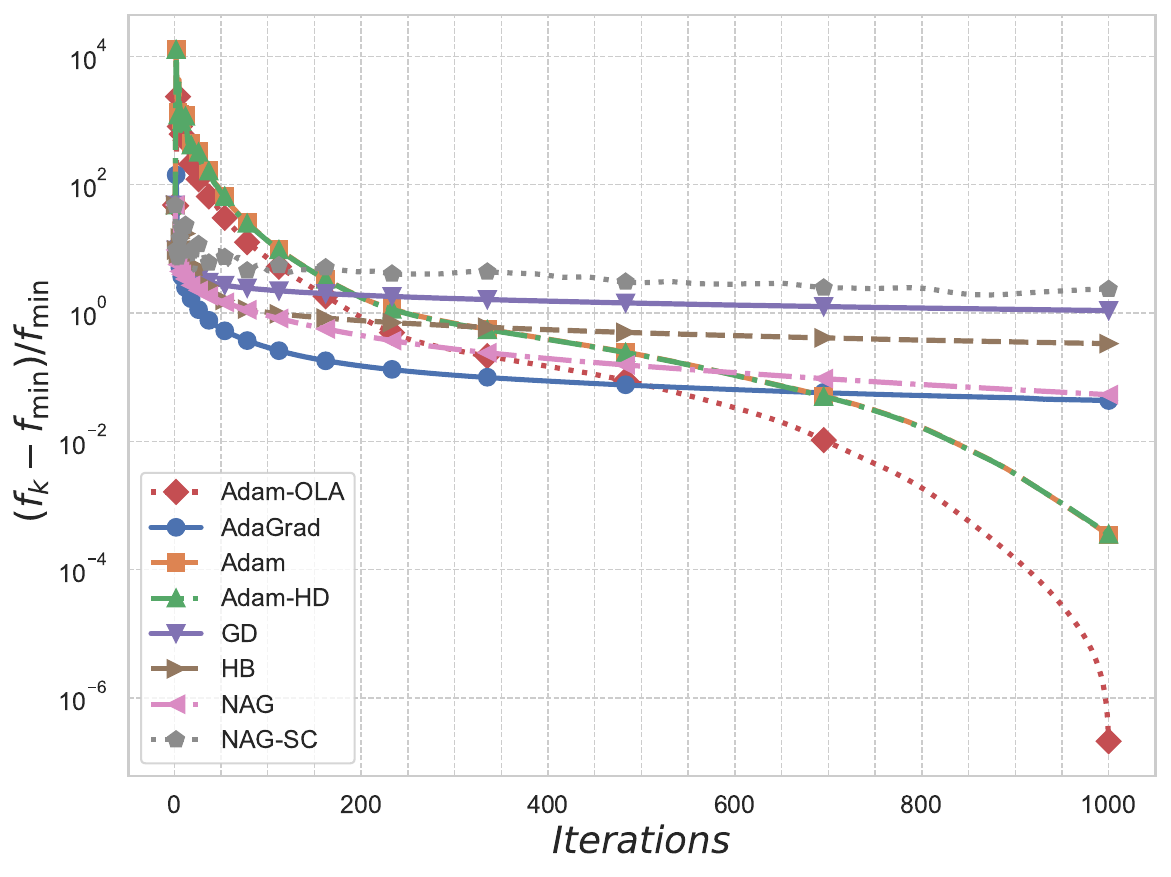}
        \caption{\texttt{w3a}}
        % \label{fig:euler-nfe}
    \end{subfigure}
    \caption{Comparison of different optimizers on smooth SVM: Function value versus iteration. Here, $f_{\min}$ denotes the minimum function value achieved across all iterations for each optimizer.}
    % \label{fig:nfe}
\end{figure}

% \begin{figure}[htbp]
%     \centering
%     \begin{subfigure}[b]{0.49\textwidth}
%         \centering
%         \includegraphics[width=\linewidth]{fig/exp_logistics_a8a.pdf}
%         \caption{\texttt{a8a}}
%         % \label{fig:ode-nfe}
%     \end{subfigure}
%     \hfill
%     \begin{subfigure}[b]{0.49\textwidth}
%         \centering
%         \includegraphics[width=\linewidth]{fig/exp_logistics_a8a.pdf}
%         \caption{\texttt{w8a}}
%         % \label{fig:euler-nfe}
%     \end{subfigure}
%     \caption{Comparison of different optimizers on logistic regression: Function value versus iteration. Here, $f_{\min}$ denotes the minimum function value achieved across all iterations for each optimizer.}
%     % \label{fig:nfe}
% \end{figure}

\textbf{Data Synthetic Setting for L2O.}
The data is synthetically generated. We first sample a sparse ground truth vector $x^\star \in \mathbb{R}^d$ with a prescribed sparsity level $s$, and then sample $W \in \mathbb{R}^{n \times d}$ with standard normal entries. The binary labels are generated via

$$
y_i = \mathbf{1}_{\{w_i^\top x^\star \geq 0\}}, \quad i = 1, \dots, n.
$$

A small proportion of labels are flipped to simulate noise.

\textbf{Architectures of L2O Optimizers.}
We now provide two examples of learned optimizers formulated within this framework, drawing from seminal works in the field. These learned optimizers typically output a direct parameter update $U_k$ such that $x_{k+1} = x_k + U_k$. To fit the continuous-time dynamical system framework where $x_{k+1} = x_k - h\mathcal{A}(\mathbf{w}_{opt}, x_k, t_k)$, we define $\mathcal{A}(\mathbf{w}_{opt}, x_k, t_k) = -U_k/h$. Here, $\mathbf{w}_{opt}$ denotes the parameters of the learned optimizer itself, $x_k$ are the parameters being optimized, and $h$ is the discretization step size from the underlying ODE.

\textbf{LSTM-based Optimizer.}
The influential work by Andrychowicz et al. \cite{l2o-dm} introduced an optimizer based on a Long Short-Term Memory (LSTM) network, which we denote as $m_{\mathbf{w}_{opt}}$. This optimizer operates coordinate-wise, meaning a small, shared-weight LSTM is applied to each parameter (coordinate) of the function $f(x)$ being optimized. For each coordinate, the LSTM takes the corresponding component of the gradient $\nabla f(x(t))$ and its own previous state, $\text{state}(t)$, as input to compute the parameter update component $U(t) = m_{\mathbf{w}_{opt}}(\nabla f(x(t)), \text{state}(t))$. The term $\text{state}(t)$ for each coordinate's LSTM, typically a multi-layer LSTM (e.g., two layers as used in the paper), consists of a tuple of (cell state, hidden state) pairs for each layer, i.e., $((c_{t,1}, h_{t,1}), (c_{t,2}, h_{t,2}))$ for a two-layer LSTM. These states allow the optimizer to accumulate information over the optimization trajectory, akin to momentum. The function $\mathcal{A}$ is then defined as
\begin{equation}
\mathcal{A}(\mathbf{w}_{opt}, x(t), t) = -\frac{1}{h} m_{\mathbf{w}_{opt}}(\nabla f(x(t)), \text{state}(t)).
\label{eq:lstm_optimizer}
\end{equation}
Here, $\mathbf{w}_{opt}$ are the learnable weights of the shared LSTM optimizer.

\textbf{RNNprop Optimizer.}
Building on similar principles, Lv et al. \cite{rnnprop} proposed the RNNprop optimizer. This optimizer also typically uses a coordinate-wise multi-layer LSTM (e.g., two-layer) as its core recurrent unit. Before the gradient information $\nabla f(x(t))$ is fed to the RNN, it undergoes a preprocessing step, $\mathcal{P}$. This preprocessing involves calculating Adam-like statistics, such as estimates of the first and second moments of the gradients, $s(t) = (\hat{m}(t), \hat{v}(t))$, which are then used to normalize the current gradient and provide historical context. The preprocessed features, $\mathcal{P}(\nabla f(x(t)), s(t))$, along with the RNN's previous state, $\text{state}(t)$, are input to the RNN. Similar to the LSTM-optimizer described above, $\text{state}(t)$ for each coordinate's RNN consists of the (cell state, hidden state) tuples for each of its layers. The output of the RNN is then passed through a scaled hyperbolic tangent function to produce the final update $U(t)$. Let this entire update-generating function be $U_{\mathbf{w}_{opt}}(\nabla f(x(t)), s(t), \text{state}(t))$. The corresponding $\mathcal{A}$ function is
\begin{equation}
\mathcal{A}(\mathbf{w}_{opt}, x(t), t) = -\frac{1}{h} U_{\mathbf{w}_{opt}}(\nabla f(x(t)), s(t), \text{state}(t)),
\label{eq:rnnprop_optimizer}
\end{equation}
where $U_{\mathbf{w}_{opt}}(\cdot)$ can be more specifically written as $\alpha \tanh(\text{RNN}(\mathcal{P}(\nabla f(x(t)), s(t)), \text{state}(t); \mathbf{w}_{opt}))$. The parameters $\mathbf{w}_{opt}$ encompass those for the preprocessing module $\mathcal{P}$ and the RNN, and $\alpha$ is a scaling hyperparameter.

\end{document}